\documentclass{article}

 \usepackage[preprint]{neurips_2026}

\usepackage[utf8]{inputenc} % allow utf-8 input
\usepackage[T1]{fontenc}    % use 8-bit T1 fonts
\usepackage{hyperref}       % hyperlinks
\usepackage{url}            % simple URL typesetting
\usepackage{booktabs}       % professional-quality tables
\usepackage{amsfonts}       % blackboard math symbols
\usepackage{nicefrac}       % compact symbols for 1/2, etc.
\usepackage{microtype}      % microtypography
\usepackage{xcolor}         % colors

\usepackage{array}

\usepackage{threeparttable}
\usepackage{adjustbox}
\usepackage{amsmath}
\usepackage{cleveref}

\crefname{appendix}{appendix}{appendices}
\Crefname{appendix}{Appendix}{Appendices}

\usepackage{booktabs}
\usepackage{graphicx}
\usepackage{subcaption}

\usepackage{wrapfig}
\usepackage{amsthm}
\newtheorem{proposition}{Proposition}[section]

\usepackage{booktabs}
\usepackage{multirow}
\usepackage{makecell}
\usepackage{threeparttable}
\usepackage{array}
\usepackage{graphicx}
\usepackage[table]{xcolor}
\definecolor{bestblue}{RGB}{218,232,246}
\definecolor{secondgreen}{RGB}{226,246,232}
\usepackage{arydshln}

\title{DiM\textsuperscript{3}: Bridging Multilingual and Multimodal Models via Direction- and Magnitude-Aware Merging}

\author{
Zijing Wang$^{1}$ 
Mingyang Wang$^{2,3}$ 
Ercong Nie$^{4}$ 
Yongkang Liu$^{1}$\thanks{Corresponding authors.} \hspace{0.3em}
Shi Feng$^{1}$ \\
\textbf{Mengjie Zhao$^{1}$ }
\textbf{Daling Wang$^{1}$\footnotemark[1] }
\textbf{Xiaocui Yang$^{1}$ }
\textbf{and Hinrich Schütze$^{2,3}$} \\
$^{1}$Northeastern University, China;
$^{2}$CIS, LMU Munich, Germany \\
$^{3}$Munich Center for Machine Learning (MCML), Germany \\
$^{4}$Shanghai Jiao Tong University, China \\
\texttt{wzj1718@gmail.com}
}

\begin{document}

\maketitle
% \begingroup
% \renewcommand{\thefootnote}{\ensuremath{\dagger}}
% \footnotetext{Corresponding authors.}
% \endgroup
% \setcounter{footnote}{0}

\begin{abstract}
Towards more general and human-like intelligence, large language models should seamlessly integrate both multilingual and multimodal capabilities; however, extending an existing multimodal model to many languages typically requires expensive multilingual multimodal data construction and repeated end-to-end retraining. 
We study a training-free alternative: injecting multilingual capability into an existing multimodal model by composing residual updates in the shared language model backbone. The key challenge is that multilingual and multimodal updates are heterogeneous, reflecting different functional roles in the shared model. To address this, we propose \textbf{Di}rection- and \textbf{M}agnitude-aware \textbf{M}ultilingual \textbf{M}ultimodal merging (\textbf{DiM\textsuperscript{3}}), which selectively composes the two updates at each parameter dimension while preserving the original vision encoder and multimodal projector. Experiments on multilingual benchmarks in both text-only and vision-language settings, covering 57 languages across LLaVA- and Qwen-based backbones, show that DiM\textsuperscript{3} consistently outperforms existing merging baselines, substantially improves multilingual performance over the original multimodal model, and remains competitive with dedicated multilingual multimodal fine-tuning while largely retaining general multimodal ability. We further show that DiM\textsuperscript{3} can be directly applied to already trained multilingual multimodal models and still yield additional gains. Further interpretability analysis shows that DiM\textsuperscript{3} primarily reshapes intermediate-layer semantic representations, strengthening cross-lingual alignment under both text-only and multimodal inputs while preserving higher-layer task-sensitive structure. Our repository is on~\url{https://github.com/wzj1718/DiM3}.

\end{abstract}

\section{Introduction}

Multimodal large language models (MLLMs)~\cite{singh2025openai,team2023gemini,wang2025internvl3,yang2025qwen3} have demonstrated remarkable capabilities across a wide range of multimodal tasks, particularly in high-resource languages (HRLs).
However, their performance degrades substantially in low-resource languages (LRLs)~\cite{pfeiffer2022xgqa,changpinyo2023maxm,schneider2024m5,mogrovejo2024cvqa},
where linguistic resources and multimodal training data are scarce.
A key reason is that existing MLLMs remain highly English-centric, with pretraining dominated by English data and only limited multilingual exposure~\cite{yue2024pangea,sunparrot,geigle2024mblip}.
Such limited multilingual exposure is insufficient to support reliable reasoning and knowledge generalization in low-resource settings, leading to persistent performance gaps.

To improve MLLMs performance in LRLs, prior work typically relies on supervised fine-tuning (SFT) or continual pre-training (CPT)~\cite{tao-etal-2024-unlocking}. SFT adapts models to downstream tasks using large amounts of labeled data, while CPT leverages large-scale target-language corpora to enhance language modeling and fluency~\cite{nguyen-etal-2024-culturax}. Despite their effectiveness, both approaches require substantial amounts of data, either labeled or unlabeled, and extensive computational resources, making them costly and difficult to scale. These limitations motivate more efficient alternatives that can enhance multilingual capability without data-intensive retraining.

Model merging~\cite{wortsman2022model,zhu2024model,liu2025sens,lu2024twin} has emerged as a promising direction, enabling the integration of complementary capabilities from separately trained models without costly large-scale training. Existing methods typically represent the parameter differences between a fine-tuned model and a shared base model as task vectors~\cite{ilharcoediting} and combine them through importance-weighted aggregation~\cite{yu2024language,yadav2023ties,luo2025sizedoesfitall}. While effective for relatively homogeneous model merging settings, these approaches struggle when task vectors are heterogeneous, differing in modality, geometric structure, and training origin, as well as functional roles. In particular, multilingual updates primarily enhance cross-lingual semantic alignment~\cite{hammerl2024understanding,xue2021mt5,gao-etal-2024-multilingual}, whereas multimodal updates adapt the backbone for visual grounding~\cite{li2023blip,liu2023visual}. Naively combining such mismatched updates can therefore lead to interference and degraded performance.

To address this limitation, we propose \textbf{DiM\textsuperscript{3}}, a direction- and magnitude-aware merging strategy tailored to heterogeneous multilingual multimodal composition. Our key insight is that task vectors encode two complementary signals~\cite{salimans2016weight,liu2024dora}: \textbf{magnitude}, reflecting the strength of parameter updates, and \textbf{direction}, capturing their functional reorientation relative to the base model.
This decomposition is crucial in heterogeneous settings, where multilingual and multimodal updates may differ not only in update scale, but also in their functional directions. 
Instead of assuming global compatibility, \textbf{DiM\textsuperscript{3}} performs \emph{selective composition} rather than uniform aggregation, adaptively assigning their contributions based on local geometric importance in parameter space. By jointly leveraging magnitude and direction, the method strengthens cross-lingual semantic alignment where multilingual signals are salient, while preserving multimodal grounding where multimodal updates dominate, enabling training-free multilingual transfer without compromising visual understanding.

We evaluate \textbf{DiM\textsuperscript{3}} on multilingual text-only and vision-language benchmarks covering 57 languages and multiple backbone families. Our method consistently outperforms existing merging baselines and substantially improves over the original multimodal model, while remaining competitive with dedicated multilingual multimodal fine-tuning despite requiring no additional training. Furthermore, DiM\textsuperscript{3} can be applied to already trained multilingual multimodal models to yield additional gains, while maintains strong performance on general multimodal benchmarks. Analysis of internal representations shows that these improvements arise from reshaping intermediate-layer semantics, leading to stronger cross-lingual alignment under both text-only and multimodal inputs.

Our contributions are threefold: \textbf{(1)} We formulate multilingual multimodal transfer as a training-free residual composition problem on a shared backbone, and demonstrate that multilingual and multimodal task vectors exhibit fundamental heterogeneity; to the best of our knowledge, this setting has not been previously explored. \textbf{(2)} We propose a direction- and magnitude-aware merging rule that selectively composes multilingual and multimodal updates at the parameter level, going beyond naive addition to better capture their distinct functional roles. \textbf{(3)} We demonstrate that this approach substantially improves multilingual multimodal performance while largely preserving general multimodal ability, can be directly applied to already trained multilingual multimodal models for further gains, and achieves these improvements by strengthening cross-lingual alignment in intermediate-layer representations.

\section{Related Work}

\paragraph{Multilingual Representation in LLMs.}

Recent work suggests that multilingual capability in LLMs is supported not only by token coverage, but also by cross-lingual alignment in internal representations~\cite{chang2022geometry,lopo2025language,wendler2024llamas,chen2026understanding}. Such alignment has been linked to stronger cross-lingual transfer and multilingual generalization~\cite{huang2025neurons,bu2026language}. Layer-wise analyses further indicate that middle layers tend to encode more transferable, language-agnostic information, whereas upper layers are more closely tied to language-specific generation~\cite{bu2025alignx,zhao2024large}. These results imply that multilingual fine-tuning produces non-uniform, functionally structured residual changes in the shared backbone.

\paragraph{Multimodal LLMs and Multilingual Multimodal Models.}
MLLMs are typically built from a vision encoder, a multimodal projector, and a pretrained language model, with much of multimodal reasoning mediated by the shared language backbone~\cite{alayrac2022flamingo,wang2026plam,zheng2025unveiling}. Existing efforts to improve multilingual capability in MLLMs therefore mainly rely on additional training, including multilingual instruction tuning, large-scale multilingual multimodal data construction, and language-aware visual instruction learning~\citep{rasheed2025palo,yue2024pangea,sunparrot,geigle2025centurio}. Recent studies also show that multimodal training can degrade multilingual fidelity, for example by inducing English-biased outputs under visual inputs, and address this issue through multilingual regularization or layer-specific enhancement~\citep{pikabea2025breaking,fan2025language}. While effective, these approaches remain training- and data-intensive, motivating our study of whether training-free parameter composition can inject multilingual capability into an existing multimodal model.

\paragraph{Model Merging.}
Model merging aims to compose capabilities from independently adapted checkpoints directly in parameter space~\cite{yang2026model,wang2025more,wang2025scaling,he2024localize,chengwhoever,du2025adamms}. Early work shows that fine-tuning residuals can be treated as task vectors and combined through simple arithmetic~\cite{ilharcoediting}. Subsequent methods improve this paradigm by reducing redundancy and mitigating interference, for example through sparsification and rescaling~\cite{yu2024language}, sign-aware conflict resolution~\cite{yadav2023ties}, and sparse update selection~\cite{davari2024model,marczak2024magmax}. More recent work explores finer-grained compatibility modeling: PCB-Merging balances parameter importance and inter-task competition during composition~\cite{du2024parameter}, STF formulates merging through the superposition of task-specific features~\cite{qiu2025superpose} in representation space, and NeuroMerging further studies interference disentanglement at the neuron level~\cite{fang2025see}. 
However, most existing methods are designed for general multi-task composition and do not explicitly address the heterogeneous updates considered here. In our setting, multilingual adaptation and multimodal adaptation play different functional roles in a shared backbone, making uniform residual addition insufficient. Our method therefore focuses on a direction-magnitude aware composition rule tailored to multilingual multimodal transfer.

\section{Methodology}\label{sub:method}

\subsection{Problem Setup}

In our setting, we consider three models that share the same language model backbone: a base model $N$, a multilingual model $M$, and a multimodal model $V$. The multilingual model extends $N$ through continual pre-training on large-scale multilingual data to improve cross-lingual understanding and generation, while the multimodal model augments the same backbone with a vision encoder and a multimodal projector for grounded reasoning, i.e., $V = V^{\mathrm{lm}} \otimes V^{\mathrm{vis}} \otimes V^{\mathrm{proj}}$. As the backbone is the only component shared across all three models, transfer is performed in this parameter space.

Let $\theta_N$ and $\theta_M$ denote the parameters of the base model $N$ and the multilingual model $M$, and let $\theta_{V^{\mathrm{lm}}}$ denote the language model parameters inside the multimodal model $V$. We define two source residuals relative to the same base point:
\begin{equation}
\Delta_{\mathrm{ml}} = \theta_M - \theta_N,\qquad
\Delta_{\mathrm{mm}} = \theta_{V^{\mathrm{lm}}} - \theta_N.
\end{equation}
Here $\Delta_{\mathrm{ml}}$ captures multilingual adaptation and $\Delta_{\mathrm{mm}}$ captures the language-side adaptation induced by multimodal training.

Our goal is to inject multilingual capability into $V$ while preserving the visual grounding and multimodal coordination already established by its visual modules. We therefore keep the vision encoder and multimodal projector fixed, and only replace the shared language model backbone with a merged backbone $\widetilde{\theta}_{V^\mathrm{lm}}$. Under the standard residual view of fine-tuning, this reduces the problem to composing two source residuals around a common reference point, rather than retraining a new multilingual multimodal model from scratch.

\subsection{Residual Heterogeneity and Geometric Motivation}

A natural baseline is direct task arithmetic,
\begin{equation}
\theta_{\mathrm{merge}} \approx \theta_N + \Delta_{\mathrm{ml}} + \Delta_{\mathrm{mm}},
\end{equation}
but this approximation is reliable only when the two residuals are sufficiently compatible~\cite{wei2025optmerge}. In our setting, multilingual and multimodal fine-tuning are driven by different supervision signals: the former primarily improves cross-lingual semantic alignment and generation, whereas the latter reshapes the same backbone to cooperate with projected visual features. This motivates a \emph{heterogeneity hypothesis}: multilingual and multimodal residuals should not be treated as interchangeable parameter offsets, because they support different functional roles in the shared backbone.

\begin{figure}[t]
  \centering
  \includegraphics[width=\linewidth]{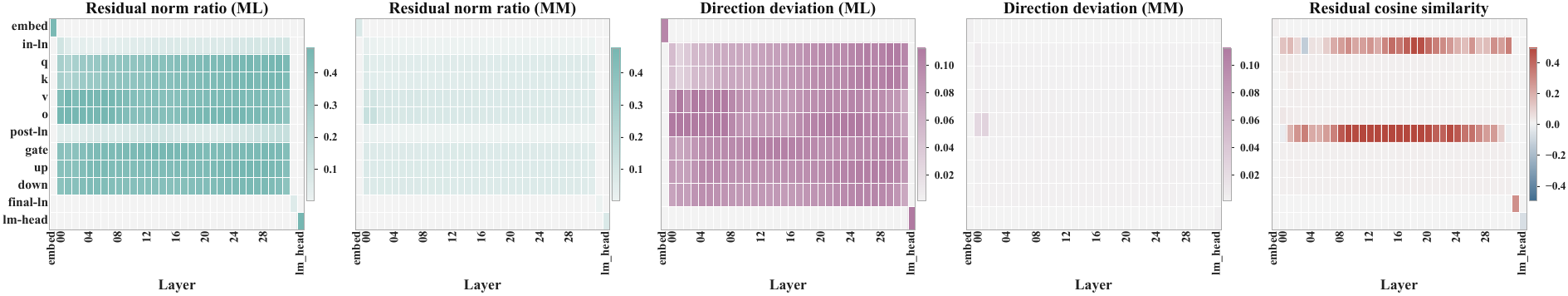}
  \captionsetup{skip=2pt} 
  \caption{\textbf{Residual heterogeneity in the shared language model backbone.} The panels show residual norm, base-relative reorientation, and cross-residual alignment for $\Delta_{\mathrm{ml}}$ and $\Delta_{\mathrm{mm}}$ across layers and modules. Together, these diagnostics reveal that multilingual and multimodal adaptations differ in both update magnitude and geometry, motivating selective rather than uniform composition.}
  \label{fig:heatmaps}
  % \vspace{-1.8em}
\end{figure}

To characterize the heterogeneity between the multilingual and multimodal updates, we study three complementary diagnostics: (1) \textit{residual norm}, which measures update magnitude; (2) \textit{base-relative reorientation}, which captures directional change relative to the base model; and (3) \textit{cross-residual alignment}, which measures how well the multilingual and multimodal updates align with each other.

For each layer weight matrix $W \in \mathbb{R}^{d_{\mathrm{out}} \times d_{\mathrm{in}}}$, we adopt a DoRA-style column-wise decomposition~\cite{liu2024dora},
representing each column as a magnitude-direction pair:
\begin{equation}
m(W)_j = \|W_{:,j}\|_2,\qquad
D(W)_{:,j} = \frac{W_{:,j}}{m(W)_j + \epsilon},
\end{equation}
where $j$ indexes the input column and $\epsilon$ is a small numerical stabilizer. Relative to the base model $N$, we define magnitude deviation and directional deviation for each source $k \in \{\mathrm{ml},\mathrm{mm}\}$ as
\begin{equation}
\delta^{\mathrm{mag}}_{k,j} = \left|m_k(j)-m_N(j)\right|,
\qquad
\delta^{\mathrm{dir}}_{k,j} = 1 - \cos\!\big(D_k(:,j),D_N(:,j)\big).
\end{equation}
Additionally, we quantify cross-residual alignment through the cosine similarity between the two source residuals at the same column.

\begin{proposition}[Residual geometry]
Ignoring $\epsilon$ for clarity, the squared residual norm of column $j$ decomposes as
\begin{equation}
\|W_{k,:,j}-W_{N,:,j}\|_2^2
=
\big(\delta^{\mathrm{mag}}_{k,j}\big)^2
+2\,m_k(j)m_N(j)\,\delta^{\mathrm{dir}}_{k,j}.
\end{equation}
\end{proposition}
\begin{proof}
Since $W_{k,:,j}=m_k(j)D_k(:,j)$, $W_{N,:,j}=m_N(j)D_N(:,j)$, and $\|D_k(:,j)\|_2=\|D_N(:,j)\|_2=1$,
\begin{equation}
\|W_{k,:,j}-W_{N,:,j}\|_2^2
=
m_k(j)^2+m_N(j)^2-2m_k(j)m_N(j)\cos\!\big(D_k(:,j),D_N(:,j)\big).
\end{equation}
Rearranging gives
\begin{equation}
  \big(m_k(j)-m_N(j)\big)^2
  +2m_k(j)m_N(j)\big(1-\cos\!\big(D_k(:,j),D_N(:,j)\big)\big)
  \end{equation}
  \begin{equation}
  =
  \big(\delta^{\mathrm{mag}}_{k,j}\big)^2
  +2\,m_k(j)m_N(j)\,\delta^{\mathrm{dir}}_{k,j},
\end{equation}
which is exactly the claimed decomposition.
\end{proof}
This decomposition shows that base-relative residual departure contains both a radial term and an angular term. In our setting, it provides a geometric basis for treating magnitude deviation and direction deviation as complementary signals: differences in norm structure favor magnitude-aware selection, while differences in reorientation favor direction-aware selection.
Figure~\ref{fig:heatmaps} confirms that $\Delta_{\mathrm{ml}}$ and $\Delta_{\mathrm{mm}}$ differ across layers and modules not only in residual norm, but also in base-relative reorientation and cross-residual alignment. This pattern indicates that the compatibility between multilingual and multimodal updates is local rather than global: some parts of backbone should remain closer to multimodal structure, whereas others can absorb stronger multilingual transfer. We therefore seek a selective merge strategy that uses both magnitude and direction signals.

\subsection{Direction- and Magnitude-Aware  Multilingual Multimodal Merging}

Following the analysis above, we merge the two source residuals in a column-wise manner within the shared language model backbone. We use column-wise aggregation because it is a natural unit for both the geometry and the function of the linear map. The magnitude-direction decomposition is defined on column vectors, and each column $W_{:,j}$ determines how input dimension $j$ contributes to all output dimensions. For each source $k\in\{\mathrm{mm},\mathrm{ml}\}$ and each column $j$, we use both $\delta^{\mathrm{mag}}_{k,j}$ and $\delta^{\mathrm{dir}}_{k,j}$ to estimate source salience. The former measures how strongly a source rewrites a column, while the latter measures how differently that column is reoriented relative to the base model.

For either branch $* \in \{\mathrm{mag},\mathrm{dir}\}$, let $\delta^*_{k,j}$ denote the corresponding deviation value. We first normalize deviations within each matrix by their relative ranks across columns:
\begin{equation}\label{eq:rank}
r^*_{k,j} = \frac{\operatorname{rank}(\delta^*_{k,j})}{d_{\mathrm{in}}}.
\end{equation}
We then convert these normalized ranks into cross-source salience scores:
\begin{equation}
s^*_{k,j} = \frac{\exp(r^*_{k,j})}{\sum_{k'} \exp(r^*_{k',j})}.
\end{equation}
This yields a relative importance score based on within-source column salience rather than absolute deviation scale. As suggested by Figure~\ref{fig:heatmaps}, the multilingual and multimodal residuals exhibit substantial cross-source heterogeneity across layers and modules, both in residual norm and direction deviation. In this setting, raw deviations are not directly comparable across the two sources: a large raw value may reflect a broader deviation range for one source in a given matrix, rather than stronger relative salience at the current column. We therefore compare the two sources through within-matrix ranks, which preserve relative ordering inside each source while reducing sensitivity to source- and module-specific scale variation. Rank normalization also makes the comparison less sensitive to a small number of extreme entries. We validate this design through salience-estimator ablations in Section~\ref{sec:ablation}.

In the two-source case, the softmax can be written as
\begin{equation}
s^*_{\mathrm{ml},j}
=
\sigma\!\left(r^*_{\mathrm{ml},j}-r^*_{\mathrm{mm},j}\right),\qquad
s^*_{\mathrm{mm},j}=1-s^*_{\mathrm{ml},j},
\end{equation}
where $\sigma(x)=1/(1+e^{-x})$. Hence the gate is driven by the salience gap between the multilingual and multimodal residuals, rather than by absolute update scale.

We aggregate the two salience branches by simple averaging:
\begin{equation}\label{eq:avg}
\omega_{k,j}=\frac{1}{2}\left(s^{\mathrm{mag}}_{k,j}+s^{\mathrm{dir}}_{k,j}\right).
\end{equation}
This symmetric design treats magnitude- and direction-based evidence equally by default, and is later validated in the ablation study. It can also be interpreted as the closed-form solution of the following consensus objective:
% =====
\begin{equation}
\omega_j
=
\arg\min_{\omega\in\Delta^2}\ 
\|\omega-s^{\mathrm{mag}}_j\|_2^2+\|\omega-s^{\mathrm{dir}}_j\|_2^2,
\end{equation}
where $\Delta^2=\{\omega\in\mathbb{R}^2:\omega_1+\omega_2=1,\ \omega_i\ge 0\}$. Thus a column receives more multilingual mass when multilingual salience is high in either branch, and remains closer to the multimodal update when multimodal salience dominates. The role of this symmetric aggregation is examined in Section~\ref{sec:ablation}.

The merged parameter is then computed column-wise as
\begin{equation}
\widetilde{W}_{:,j}
=
W_{N,:,j}
+
\sum_{k\in\{\mathrm{mm},\mathrm{ml}\}}
\omega_{k,j}\,\Delta_k^{(W)}{}_{:,j},
\end{equation}

For one-dimensional parameters such as normalization weights or bias terms, the direction-magnitude decomposition is not available. We therefore use absolute base-relative deviation,
\begin{equation}
\delta^{1\mathrm{d}}_{k,i}=|w_{k,i}-w_{N,i}|,
\end{equation}
derive element-wise salience scores in the same rank-based manner, and merge by
\begin{equation}
\widetilde{w}_i
=
w_{N,i}
+
\sum_{k\in\{\mathrm{mm},\mathrm{ml}\}}
\gamma_{k,i}\,\Delta_k^{(w)}{}_i.
\end{equation}
After merging all shared backbone parameters, we write $\widetilde{\theta}_{V^\mathrm{lm}}$ back into the multimodal anchor while leaving the vision encoder and projector unchanged. The final merged multimodal model is
\begin{equation}
\widetilde{V}
=
\big(
\theta_{V^{\mathrm{vis}}},
\theta_{V^{\mathrm{proj}}},
\widetilde{\theta}_{V^\mathrm{lm}}
\big).
\end{equation}

\begin{table*}[!h]
\vspace{-1em}
\centering
\captionsetup{skip=2pt}
\caption{Role assignment of models in each shared-backbone setting.}
\label{tab:backbone_settings}
\small
\renewcommand{\arraystretch}{1.15}
\setlength{\tabcolsep}{6pt}
\begin{adjustbox}{max width=\textwidth}
\begin{tabular}{l l l l l}
\toprule
\textbf{Family} & \textbf{Backbone} & \textbf{MM anchor} & \textbf{Multilingual expert} & \textbf{Extra references} \\
\midrule
LLaMA  & LLaMA-2-7B~\cite{touvron2023llama}  & LLaVA-v1.5-7B~\cite{liu2023visual}   & Emma-500~\cite{ji2024emma}  & PLAST-7B~\cite{fan2025language}; Palo-7B/13B~\cite{rasheed2025palo} \\
Qwen2  & Qwen2-7B-Instruct~\cite{qwen2}  & Pangea-7B~\cite{yue2024pangea}    & Qwen2-7B-Multilingual-RP  & Pangea-7B \\
Qwen3  & Qwen3-8B~\cite{yang2025qwen3}   & Qwen3-VL-8B-Instruct~\cite{bai2025qwen3} & AfriqueQwen-8B~\cite{yu2026afriquellm}  & Qwen3-VL-8B-Instruct \\
\bottomrule
\end{tabular}
\end{adjustbox}
\vspace{-2.0em}
\end{table*}

\section{Experimental Setup}

\paragraph{Baselines}
We conduct experiments under three shared-backbone settings, summarized in Table~\ref{tab:backbone_settings}. In all cases, merging is performed only on the shared language model parameters, while the visual components of the multimodal model are kept fixed.

For \textbf{merging method baselines}, we compare DiM\textsuperscript{3} with representative training-tree merging approaches, including Task Arithmetic~\cite{ilharcoediting}, DARE~\cite{yu2024language}, TIES-Merging~\cite{yadav2023ties}, Breadcrumbs~\cite{davari2024model}, PCB-Merging~\cite{du2024parameter}, STF~\cite{qiu2025superpose}, and NeuroMerging~\cite{fang2025see}. These methods span several major directions in model merging, including direct task vector composition, sparsity-based selection, conflict resolution, and more structured compatibility modeling. This comparison allows us to test whether DiM\textsuperscript{3} provides advantages over existing merging strategies in multilingual and multimodal settings. Details on these baselines are provided in Appendix~\ref{app:exp}.

For \textbf{model baselines}, we report the performance of the backbone-compatible source models and multilingual multimodal reference systems used in each setting. In the LLaMA-based setup, these include the base model LLaMA-2-7B, the multilingual model Emma-500, the multimodal anchor LLaVA-v1.5-7B, PLAST-7B~\cite{fan2025language}, and Palo-7B/13B~\cite{rasheed2025palo}, where PLAST-7B is a parameter-efficient multilingual enhancement of LLaVA-v1.5-7B and Palo is a multilingual multimodal model fine-tuned from the same backbone. In the Qwen3-based setting, we further report Qwen3-VL-8B-Instruct~\cite{bai2025qwen3} and AfriqueQwen-8B~\cite{yu2026afriquellm} as the multimodal and multilingual source models, respectively. To examine whether DiM\textsuperscript{3} can also be directly applied to an existing multilingual multimodal model, we additionally conduct merging experiments based on Pangea-7B~\cite{yue2024pangea}, using it as a pretrained multilingual multimodal checkpoint in the Qwen2 family.

\paragraph{Benchmarks}
We evaluate DiM\textsuperscript{3} on three categories of benchmarks.

For \textbf{multilingual text-only tasks}, we consider XCOPA~\cite{ponti2020xcopa}, XStoryCloze~\cite{lin2022few}, and XNLI~\cite{conneau2018xnli}. On these benchmarks, we evaluate the shared language model backbones of LLaMA-2-7B, Emma-500, LLaVA-v1.5-7B, and the models produced by different merging methods. This setting isolates multilingual semantic transfer in pure text scenarios and tests whether merging preserves the multilingual gains of the source model. Detailed benchmark settings are provided in Appendix~\ref{app:exp}.
We also list all languages used in our experiments and their corresponding ISO 639-1/2 codes in Appendix Table~\ref{tab:language_codes}.

For \textbf{multilingual multimodal tasks}, we use xMMMU~\cite{yue2024pangea}, MaXM~\cite{changpinyo2023maxm}, MaRVL~\cite{liu2021visually}, xGQA~\cite{pfeiffer2022xgqa}, CVQA~\cite{mogrovejo2024cvqa}, and Afri-MCQA~\cite{tonja2026afri}. Together, these benchmarks are used to evaluate multilingual visual understanding and cross-lingual multimodal reasoning. In the LLaMA-based setting, we report results on all six benchmarks. We further include Afri-MCQA in the Qwen2- and Qwen3-based shared-backbone settings to examine DiM\textsuperscript{3} under additional backbone families.

For \textbf{general multimodal tasks}, we report results on MMStar~\cite{chen2024we}, MMMU~\cite{yue2024mmmu}, and SEED-Bench-2-Plus~\cite{li2024seed}. We use these benchmarks in the main LLaMA-based setting to assess whether multilingual transfer degrades the general multimodal competence of the anchor model.

All evaluation results are obtained using the \textbf{lm-evaluation-harness}~\cite{biderman2024lessons} and \textbf{lmms-eval}~\cite{zhang2025lmms} libraries under the same settings to ensure fair comparison.
For the merging baselines, hyperparameters are chosen by grid search on validation sets, and we report the best configuration for each method.

\section{Main Results}\label{sec:results}

\subsection{Multilingual Capability Transfer in Text-Only Benchmarks}

\paragraph{Comparison with Existing Merging Methods on Multilingual Text Benchmarks.}

We first evaluate whether DiM\textsuperscript{3} can transfer multilingual capability to the shared language backbone under text-only inputs. Following Section~\ref{sub:method}, we compose the multilingual residual from Emma-500 with the language-backbone residual of LLaVA-v1.5-7B around the common base model LLaMA-2, while keeping the visual modules fixed. For text-only evaluation, we extract the language model component from each merged model and evaluate it independently.
\begin{wraptable}{r}{0.40\columnwidth}
\vspace{-0.6em}
\centering
\scriptsize
\setlength{\tabcolsep}{4.8pt}
\renewcommand{\arraystretch}{0.78}
\captionsetup{skip=2pt}
\caption{Results on multilingual text-only tasks. 
\protect\colorbox{bestblue}{Best} and 
\protect\colorbox{secondgreen}{second-best} results within each dataset are highlighted.}
\label{tab:text_only_results}
\resizebox{0.40\columnwidth}{!}{
\begin{tabular}{lccc}
\toprule
\textbf{Models} & \textbf{XCOPA} & \textbf{XStoryCloze} & \textbf{XNLI} \\
\specialrule{0.08em}{0.2em}{0.1em}
\multicolumn{4}{c}{\textit{language models}} \\[-2pt]
\midrule
Llama\_2\_7B    & 56.27 & 55.74 & 39.07 \\
EMMA-500        & \cellcolor{secondgreen}62.47 & \cellcolor{secondgreen}64.77 & \cellcolor{bestblue}\textbf{43.64} \\
LLaVA-v1.5-7B   & 51.98 & 49.48 & 35.55 \\
\specialrule{0.08em}{0.2em}{0.1em}
\multicolumn{4}{c}{\textit{merging methods}} \\[-2pt]
\midrule
Task Arithmetic & 56.78 & 58.41 & 40.07 \\
DARE            & 56.55 & 56.81 & 39.14 \\
TIES-Merging    & 57.44 & 58.57 & 40.47 \\
Breadcrumbs     & 56.98 & 57.49 & 40.42 \\
PCB-Merging     & 55.55 & 54.18 & 37.14 \\
STF             & 55.40 & 56.25 & 39.14 \\
NeuroMerging    & 56.09 & 56.54 & 40.01 \\
DiM\textsuperscript{3}            & \cellcolor{bestblue}\textbf{63.38} & \cellcolor{bestblue}\textbf{64.82} & \cellcolor{secondgreen}42.93 \\
\bottomrule
\end{tabular}
}
\vspace{-1.8em}
\end{wraptable}
Table~\ref{tab:text_only_results} summarizes the average multilingual score $\mathrm{Avg}_{\mathrm{mul}}$ on XCOPA, XStoryCloze, and XNLI, while Appendix Tables~\ref{tab:xcopa_xstorycloze_methods}, \ref{tab:xnli_methods}, \ref{tab:xcopa_xstorycloze_models}, and \ref{tab:xnli_models} report the full language-wise results. Our method achieves the best average multilingual score on all three benchmarks, consistently outperforming existing merging baselines. This shows that direction-magnitude-aware merging is better suited to composing multilingual and multimodal residuals than prior parameter-merging rules.

\paragraph{Comparison with Base Models on Multilingual Text Benchmarks.}
The same pattern holds when compared with backbone-compatible source models. Relative to LLaVA-v1.5-7B, DiM\textsuperscript{3} yields large gains on all three text benchmarks, confirming that multilingual capability is effectively transferred into the shared language backbone. At the same time, it remains close to Emma-500 overall, and even surpasses it in $\mathrm{Avg}_{\mathrm{mul}}$ on XCOPA and XStoryCloze. Overall, these results show that our method can effectively inherit multilingual capability from the multilingual source model while maintaining compatibility with the original LLaVA backbone.

\begin{figure*}[!htbp]
\centering
\begin{minipage}[t]{0.58\textwidth}
\vspace{0pt}
\centering
\scriptsize
\setlength{\tabcolsep}{2.8pt}
\renewcommand{\arraystretch}{0.88}
\captionsetup{skip=2pt}
\captionof{table}{Results on multilingual multimodal tasks.}
\label{tab:vl_results}
\resizebox{\textwidth}{!}{
\begin{tabular}{lcccccc}
\toprule
\textbf{Models} & \textbf{XMMMU} & \textbf{MaXM} & \textbf{MaRVL} & \textbf{XGQA} & \textbf{CVQA} & \textbf{Afri-MCQA} \\
\midrule
LLaVA-v1.5-7B   & 31.48 & 20.10 & 49.78 & 30.62 & 39.32 & \cellcolor{secondgreen}24.61 \\
Task Arithmetic & 32.67 & 22.27 & 49.81 & 36.54 & \cellcolor{secondgreen}40.47 & 24.39 \\
DARE            & 31.38 & 19.54 & 49.76 & 31.51 & 38.42 & 24.32 \\
TIES-Merging    & 32.27 & 24.23 & 49.71 & 36.64 & 39.22 & 23.84 \\
Breadcrumbs     & 29.92 & 21.92 & 49.64 & 36.03 & 34.24 & 21.00 \\
PCB-Merging     & 27.95 & 0.06 & -- & -- & 7.80 & 0.00 \\
STF             & 32.05 & 11.09 & \cellcolor{secondgreen}50.87 & 10.87 & 34.59 & 23.86 \\
NeuroMerging    & \cellcolor{secondgreen}33.47 & 20.53 & 49.84 & 20.98 & 39.38 & 23.67 \\
Palo-7B         & 31.30 & 16.45 & 50.73 & \cellcolor{secondgreen}36.93 & 37.00 & 24.47 \\
Palo-13B        & 31.30 & \cellcolor{secondgreen}28.68 & 49.70 & 30.60 & 39.59 & 23.95 \\
% \rowcolor{neuripsblue}
DiM\textsuperscript{3}-7B            & \cellcolor{bestblue}\textbf{33.85} & \cellcolor{bestblue}\textbf{33.44} & \cellcolor{bestblue}\textbf{62.91} & \cellcolor{bestblue}\textbf{48.04} & \cellcolor{bestblue}\textbf{43.78} & \cellcolor{bestblue}\textbf{26.31} \\
\bottomrule
\end{tabular}
}
\end{minipage}
\hfill
\begin{minipage}[t]{0.37\textwidth}
\vspace{0pt}
\centering
\vspace{-0.1\baselineskip}
\includegraphics[width=\textwidth]{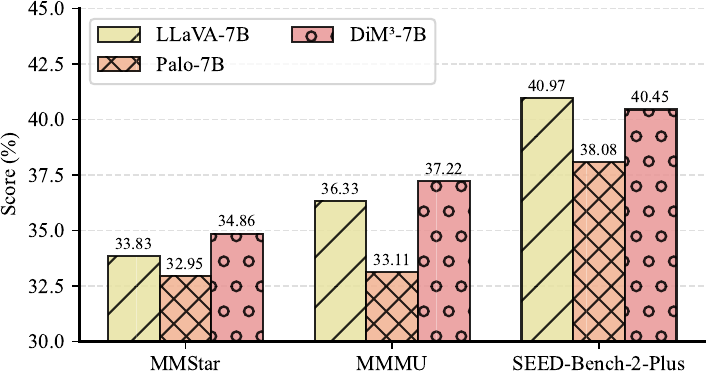}
\captionsetup{skip=3pt}
\captionof{figure}{Results on three general multimodal benchmarks.}
\label{fig:general_mm_results}
\end{minipage}
\vspace{-1.8em}
\end{figure*}

\begin{table*}[!htbp]
\centering
\scriptsize
\setlength{\tabcolsep}{2pt}
\renewcommand{\arraystretch}{1.00} 
\captionsetup{skip=2pt} 
\caption{Afri-MCQA results under the Qwen2- and Qwen3-based shared-backbone settings.}
\label{tab:qwen_afri}
\resizebox{\textwidth}{!}{
\begin{tabular}{llccccccccccccccccc}
\toprule
\textbf{Backbone} & \textbf{Model} 
& \textbf{Am} & \textbf{Ha} & \textbf{Ig} & \textbf{Lug} 
& \textbf{Om} & \textbf{Rw} & \textbf{Ki} & \textbf{So} 
& \textbf{Ti} & \textbf{Ak} & \textbf{Yo} & \textbf{Tn} 
& \textbf{Ny} & \textbf{Zu} & \textbf{Sot} & \textbf{Ln} & \textbf{Avg} \\
\midrule

\multirow{2}{*}{Qwen2} 
& Pangea-7B 
& \cellcolor{bestblue}\textbf{0.357} & 0.271 & 0.352 & 0.320 
& 0.130 & 0.370 & 0.253 & 0.310 
& 0.287 & 0.332 & 0.380 & 0.185 
& 0.265 & 0.198 & 0.131 & 0.389 & 0.283 \\

& DiM\textsuperscript{3} 
& 0.317 & \cellcolor{bestblue}\textbf{0.342} & \cellcolor{bestblue}\textbf{0.407} & \cellcolor{bestblue}\textbf{0.425} 
& \cellcolor{bestblue}\textbf{0.320} & \cellcolor{bestblue}\textbf{0.480} & \cellcolor{bestblue}\textbf{0.376} & \cellcolor{bestblue}\textbf{0.345} 
& \cellcolor{bestblue}\textbf{0.347} & \cellcolor{bestblue}\textbf{0.388} & \cellcolor{bestblue}\textbf{0.390} & \cellcolor{bestblue}\textbf{0.315} 
& \cellcolor{bestblue}\textbf{0.390} & \cellcolor{bestblue}\textbf{0.333} & \cellcolor{bestblue}\textbf{0.331} & \cellcolor{bestblue}\textbf{0.419} & \cellcolor{bestblue}\textbf{0.370} \\

\noalign{\vskip 1pt}
\hdashline
\noalign{\vskip 1pt}

\multirow{2}{*}{Qwen3} 
& Qwen3-VL-8B-Instruct 
& 0.377 & 0.342 & 0.347 & 0.415 
& 0.335 & 0.460 & 0.392 & \cellcolor{bestblue}\textbf{0.405} 
& 0.357 & \cellcolor{bestblue}\textbf{0.430} & 0.420 & 0.350 
& 0.380 & 0.302 & \cellcolor{bestblue}\textbf{0.323} & \cellcolor{bestblue}\textbf{0.490} & 0.383 \\

& DiM\textsuperscript{3} 
& \cellcolor{bestblue}\textbf{0.493} & \cellcolor{bestblue}\textbf{0.498} & \cellcolor{bestblue}\textbf{0.362} & \cellcolor{bestblue}\textbf{0.450} 
& \cellcolor{bestblue}\textbf{0.445} & \cellcolor{bestblue}\textbf{0.645} & \cellcolor{bestblue}\textbf{0.407} & \cellcolor{bestblue}\textbf{0.555} 
& \cellcolor{bestblue}\textbf{0.505} & 0.388 & \cellcolor{bestblue}\textbf{0.445} & \cellcolor{bestblue}\textbf{0.420} 
& \cellcolor{bestblue}\textbf{0.480} & \cellcolor{bestblue}\textbf{0.406} & 0.277 & 0.460 & \cellcolor{bestblue}\textbf{0.452} \\

\bottomrule
\end{tabular}
}
\vspace{-2.0em}
\end{table*}

\subsection{Multilingual Transfer in Multimodal Benchmarks}

We next evaluate multilingual capability transfer under multimodal inputs. Table~\ref{tab:vl_results} summarizes the average multilingual score $\mathrm{Avg}_{\mathrm{mul}}$ on six benchmarks: xMMMU, MaXM, MaRVL, xGQA, CVQA, and Afri-MCQA, while Appendix Tables~\ref{tab:xmmmu_maxm_combined}, \ref{tab:marvl_xgqa}, \ref{tab:multilingual_results}, and \ref{tab:cvqa} report the full language-wise results. Across all six benchmarks, our method achieves the best average multilingual performance, showing that DiM\textsuperscript{3} transfers multilingual capability effectively in multimodal settings.

A notable phenomenon is that some existing merging baselines become unstable in this heterogeneous multimodal setting. This is especially evident on generative multilingual tasks such as MaXM, where several methods produce repetitive, weakly relevant, or degenerate outputs; the problem is particularly severe for PCB-Merging, which we omit from later multimodal evaluations. This failure mode suggests that successful multimodal merging must preserve the coordination between language generation and visual grounding, not merely interpolate parameters. In this regard, DiM\textsuperscript{3} is substantially more robust than prior merging strategies.

The appendix further shows that the gains are broadly distributed across languages and are particularly pronounced in lower-resource settings. Table~\ref{tab:qwen_afri} further validates this trend under the additional shared-backbone settings in Table~\ref{tab:backbone_settings}, while Appendix Table~\ref{tab:multilingual_results} reports the corresponding results of different merging methods. In the Qwen3 setting, DiM\textsuperscript{3} improves Qwen3-VL-8B-Instruct from 0.383 to 0.452 average accuracy, confirming that the method remains effective across a different backbone family. In the Pangea-based Qwen2 setting, merging Pangea-7B with Qwen2-7B-Multilingual-RP raises the average score from 0.283 to 0.370, showing that DiM\textsuperscript{3} can further improve an already fine-tuned multilingual multimodal model. Compared with dedicated multilingual multimodal fine-tuning baselines such as Palo, our method remains highly competitive while requiring no end-to-end multilingual multimodal retraining.

\begin{figure}[t]
  \vspace{-1.1em}
  \centering
  \includegraphics[width=\linewidth]{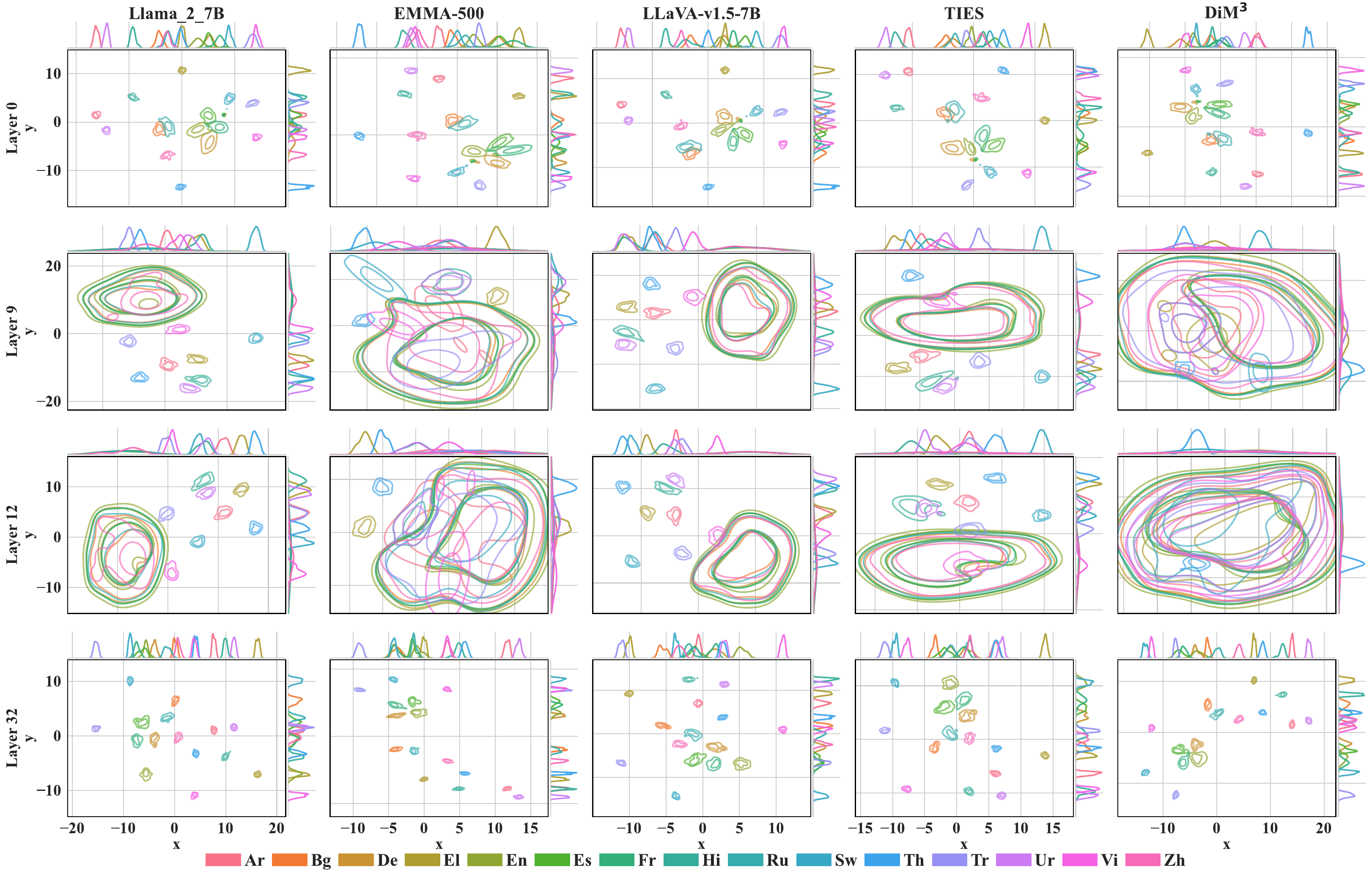}
  \captionsetup{skip=1.0pt}
  \caption{t-SNE visualizations of average-pooled hidden states on multilingual text inputs from XNLI across selected layers for LLaMA-2-7B, Emma-500, LLaVA-v1.5-7B, TIES-Merging, and DiM\textsuperscript{3}. Colors denote languages. DiM\textsuperscript{3} shows more overlapping language clusters at intermediate layers, consistent with stronger cross-lingual alignment.}
  \label{fig:text_only_tsne}
  \vspace{-1.3em}
\end{figure}

\subsection{Retaining General Multimodal Ability}

\begin{wraptable}{r}{0.4\columnwidth}
\vspace{-0.8em}
\centering
\scriptsize
\setlength{\tabcolsep}{3.2pt}
\renewcommand{\arraystretch}{0.82}
\captionsetup{skip=2pt}
\caption{Ablations on salience estimation and branch aggregation.}
\label{tab:ablation_merge_rule}
\resizebox{0.4\columnwidth}{!}{
\begin{tabular}{lccccc}
\toprule
\textbf{Variant} & \textbf{MMStar} & \textbf{MMMU} & \textbf{SEED} & \textbf{MaXM}& \textbf{MaRVL} \\
\midrule
DiM\textsuperscript{3} 
& \cellcolor{bestblue}\textbf{34.86} 
& \cellcolor{bestblue}\textbf{37.22} 
& 40.45 
& \cellcolor{bestblue}\textbf{33.44} 
& \cellcolor{bestblue}\textbf{62.91}\\
MinMax 
& 33.74 & 29.00 & 39.48 & 32.57 & 52.23\\
Ratio 
& 2.56 & 25.33 & 10.01 & 3.58 & 34.16\\
Raw 
& 33.64 & 31.00 & 39.48 & 33.00 &52.29\\
ZScore 
& \cellcolor{secondgreen}34.50 
& \cellcolor{secondgreen}36.67 
& \cellcolor{bestblue}\textbf{40.80} 
& 33.32 & 51.30\\
Dir-W 
& 34.27 
& \cellcolor{secondgreen}36.67 
& \cellcolor{secondgreen}40.71 
& 31.71 & 51.90\\
Mag-W 
& 34.26 & 36.33 & 40.40 & 33.19 & 51.03\\
Mag-Only 
& 33.89 & 34.78 & 40.10 
& \cellcolor{secondgreen}33.34 & 50.16\\
Dir-Only 
& 34.01 & 30.33 & 40.01 & 29.82 & \cellcolor{secondgreen}55.93\\
\bottomrule
\end{tabular}
}
\vspace{-1.0em}
\end{wraptable}

We further examine whether multilingual transfer compromises general multimodal ability. As shown in Figure~\ref{fig:general_mm_results}, DiM\textsuperscript{3}-7B remains comparable to the LLaVA-7B anchor, improving on MMStar and MMMU and staying close on SEED-Bench-2-Plus.
Meanwhile, it consistently outperforms the directly fine-tuned multilingual multimodal baseline Palo-7B, with scores of 34.86 vs.\ 32.95, 37.22 vs.\ 33.11, and 40.45 vs.\ 38.08, respectively. These results indicate that DiM\textsuperscript{3} better preserves the anchor model's general multimodal competence than the fine-tuned multilingual multimodal baseline. Overall, the trade-off is mild: the merged model remains comparable to LLaVA on general multimodal tasks while delivering substantially stronger multilingual gains.

\subsection{Ablations}\label{sec:ablation}

We ablate two components of the merge rule: the salience estimator in Eq.~\ref{eq:rank} and the branch aggregation rule in Eq.~\ref{eq:avg}. Table~\ref{tab:ablation_merge_rule} reports results on five representative multimodal benchmarks.
For salience estimation, we compare the proposed rank-based softmax with four alternatives: Raw, which applies softmax directly to the original deviations; ZScore, which standardizes deviations within each tensor before softmax; MinMax, which replaces rank normalization with min-max normalization; and Ratio, which assigns source weights in proportion to their raw deviations without the softmax comparison. For branch aggregation, we compare the symmetric average in Eq.~\ref{eq:avg} with four alternatives: Dir-Weighted (Dir-W) and Mag-Weighted (Mag-W), which use asymmetric combinations of the two branches, and Magnitude-Only (Mag-Only) and Direction-Only (Dir-Only), which keep only one branch.

Two conclusions emerge. First, rank-based salience is the most robust estimator: Ratio degrades sharply on all five benchmarks, while Raw and MinMax are consistently weaker; ZScore is more competitive, but still underperforms our method on four of the five tasks. Second, the joint two-branch design matters: both Magnitude-Only and Direction-Only fall short of the full model, and both asymmetric variants are slightly worse than the symmetric average.

\begin{figure}[!h]
\vspace{-1em}
  \centering
  \includegraphics[width=0.8\linewidth]{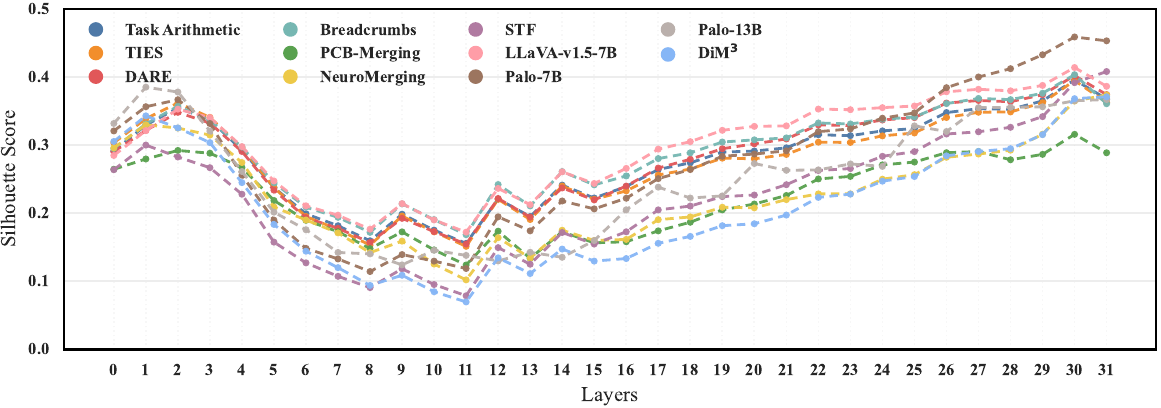}
  \captionsetup{skip=-0.05em}
  \caption{Layer-wise silhouette scores of multilingual hidden-state representations under multilingual multimodal inputs. Lower values indicate stronger cross-lingual overlap in the representation space.}
  \label{fig:xmmmu_silhouette2}
  \vspace{-2em}
\end{figure}

\section{Analysis: What Does the Merge Change?}\label{sec:analysis}

\paragraph{Hidden-State Geometry on Multilingual Text}

To understand why merging improves multilingual transfer, we analyze the hidden-state geometry of multilingual text inputs. Figure~\ref{fig:text_only_tsne} presents t-SNE visualizations of average-pooled hidden states on XNLI for LLaMA-2-7B, EMMA-500, LLaVA-v1.5-7B, TIES-Merging, and DiM\textsuperscript{3}; full-layer results and additional baselines are provided in Appendix Figures~\ref{fig:text_only_tsne_merging_methods}, \ref{fig:text_only_tsne_0-13}, \ref{fig:text_only_tsne_14-27}, and \ref{fig:text_only_tsne_28-32}.
Across models, we observe a consistent layer-wise pattern: lower layers remain largely language-dependent, reflecting lexical and surface-form information. In middle layers, however, DiM\textsuperscript{3} exhibits substantially stronger cross-lingual overlap than LLaMA-2-7B, LLaVA-v1.5-7B, and existing merging baselines. In upper layers, representations become more dispersed again across languages, indicating the reintroduction of language-sensitive structure for generation. Overall, this analysis suggests that merging primarily reshapes the intermediate semantic space, strengthening cross-lingual alignment without disrupting higher-layer task-sensitive representations.

\paragraph{Hidden-State Geometry in Multimodal Multilingual Inputs}

We next examine whether this representation-level effect persists under multimodal conditioning. 
Appendix Figure~\ref{fig:xmmmu_tsne} shows t-SNE visualizations of average-pooled hidden states from xMMMU for LLaVA-v1.5-7B and our merged model, and Figure~\ref{fig:xmmmu_silhouette2} reports the corresponding layer-wise silhouette scores for all baselines, where lower values indicate stronger cross-lingual overlap.
The multimodal results closely mirror the text-only setting. Lower layers retain clear language separation, whereas in the middle layers DiM\textsuperscript{3} yields substantially stronger cross-lingual overlap than LLaVA, indicating a more shared semantic space under visual inputs. This trend is supported quantitatively by the silhouette scores, where DiM\textsuperscript{3} consistently achieves lower values across most layers compared to LLaVA and other merging baselines.  The increase in higher layers further suggests that improved mid-layer alignment does not collapse representations into a single undifferentiated cluster.
Overall, these findings confirm that DiM\textsuperscript{3} primarily reshapes intermediate semantic representations, now under visual conditioning, providing a consistent explanation for the observed multilingual multimodal gains in Section~\ref{sec:results}.

\section{Conclusion}

In this work, we formulate multilingual multimodal composition as a training-free residual composition problem on a shared vision-language backbone, and show that multilingual and multimodal adaptations should be treated as heterogeneous residual updates rather than homogeneous parameter offsets. We propose a direction- and magnitude-aware merging strategy that selectively composes these updates while preserving the visual modules of the multimodal anchor. Experiments on multilingual text and multilingual multimodal benchmarks show that DiM\textsuperscript{3} consistently improves multilingual performance, outperforms representative merging baselines, remains competitive with dedicated multilingual multimodal fine-tuning, and largely preserves general multimodal ability. Overall, our results highlight that effective capability composition depends not only on the magnitude of parameter changes, but also on their directional alignment, with the primary impact occurring in intermediate semantic representations.

\bibliography{custom}

@inproceedings{ilharcoediting,
  title={Editing models with task arithmetic},
  author={Ilharco, Gabriel and Ribeiro, Marco Tulio and Wortsman, Mitchell and Schmidt, Ludwig and Hajishirzi, Hannaneh and Farhadi, Ali},
  booktitle={The Eleventh International Conference on Learning Representations}
}

@inproceedings{wortsman2022model,
  title={Model soups: averaging weights of multiple fine-tuned models improves accuracy without increasing inference time},
  author={Wortsman, Mitchell and Ilharco, Gabriel and Gadre, Samir Ya and Roelofs, Rebecca and Gontijo-Lopes, Raphael and Morcos, Ari S and Namkoong, Hongseok and Farhadi, Ali and Carmon, Yair and Kornblith, Simon and others},
  booktitle={International conference on machine learning},
  pages={23965--23998},
  year={2022},
  organization={PMLR}
}

@inproceedings{yue2024pangea,
  title={Pangea: A fully open multilingual multimodal llm for 39 languages},
  author={Yue, Xiang and Song, Yueqi and Asai, Akari and Kim, Seungone and de Dieu Nyandwi, Jean and Khanuja, Simran and Kantharuban, Anjali and Sutawika, Lintang and Ramamoorthy, Sathyanarayanan and Neubig, Graham},
  booktitle={The Thirteenth International Conference on Learning Representations},
  year={2024}
}

@article{tonja2026afri,
  title={Afri-MCQA: Multimodal Cultural Question Answering for African Languages},
  author={Tonja, Atnafu Lambebo and Anand, Srija and Villa-Cueva, Emilio and Azime, Israel Abebe and Alabi, Jesujoba Oluwadara and Mohamed, Muhidin A and Yadeta, Debela Desalegn and Abadi, Negasi Haile and Oppong, Abigail and Obiefuna, Nnaemeka Casmir and others},
  journal={arXiv preprint arXiv:2601.05699},
  year={2026}
}

@inproceedings{liu2025sens,
  title={Sens-merging: Sensitivity-guided parameter balancing for merging large language models},
  author={Liu, Shuqi and Wu, Han and He, Bowei and Han, Xiongwei and Yuan, Mingxuan and Song, Linqi},
  booktitle={Findings of the Association for Computational Linguistics: ACL 2025},
  pages={19243--19255},
  year={2025}
}

@article{salimans2016weight,
  title={Weight normalization: A simple reparameterization to accelerate training of deep neural networks},
  author={Salimans, Tim and Kingma, Durk P},
  journal={Advances in neural information processing systems},
  volume={29},
  year={2016}
}

@inproceedings{liu2024dora,
  title={Dora: Weight-decomposed low-rank adaptation},
  author={Liu, Shih-Yang and Wang, Chien-Yi and Yin, Hongxu and Molchanov, Pavlo and Wang, Yu-Chiang Frank and Cheng, Kwang-Ting and Chen, Min-Hung},
  booktitle={Forty-first International Conference on Machine Learning},
  year={2024}
}

@inproceedings{yu2024language,
  title={Language models are super mario: Absorbing abilities from homologous models as a free lunch},
  author={Yu, Le and Yu, Bowen and Yu, Haiyang and Huang, Fei and Li, Yongbin},
  booktitle={Forty-first International Conference on Machine Learning},
  year={2024}
}

@article{lu2024twin,
  title={Twin-merging: Dynamic integration of modular expertise in model merging},
  author={Lu, Zhenyi and Fan, Chenghao and Wei, Wei and Qu, Xiaoye and Chen, Dangyang and Cheng, Yu},
  journal={Advances in Neural Information Processing Systems},
  volume={37},
  pages={78905--78935},
  year={2024}
}

@article{yadav2023ties,
  title={Ties-merging: Resolving interference when merging models},
  author={Yadav, Prateek and Tam, Derek and Choshen, Leshem and Raffel, Colin A and Bansal, Mohit},
  journal={Advances in neural information processing systems},
  volume={36},
  pages={7093--7115},
  year={2023}
}

@inproceedings{davari2024model,
  title={Model breadcrumbs: Scaling multi-task model merging with sparse masks},
  author={Davari, MohammadReza and Belilovsky, Eugene},
  booktitle={European Conference on Computer Vision},
  pages={270--287},
  year={2024},
  organization={Springer}
}

@article{du2024parameter,
  title={Parameter competition balancing for model merging},
  author={Du, Guodong and Lee, Junlin and Li, Jing and Jiang, Runhua and Guo, Yifei and Yu, Shuyang and Liu, Hanting and Goh, Sim K and Tang, Ho-Kin and He, Daojing and others},
  journal={Advances in Neural Information Processing Systems},
  volume={37},
  pages={84746--84776},
  year={2024}
}

@inproceedings{qiu2025superpose,
  title={Superpose Task-specific Features for Model Merging},
  author={Qiu, Haiquan and Wu, You and Li, Dong and Guo, Jianmin and Yao, Quanming},
  booktitle={Proceedings of the 2025 Conference on Empirical Methods in Natural Language Processing},
  pages={4200--4214},
  year={2025}
}

@inproceedings{li2023blip,
  title={Blip-2: Bootstrapping language-image pre-training with frozen image encoders and large language models},
  author={Li, Junnan and Li, Dongxu and Savarese, Silvio and Hoi, Steven},
  booktitle={International conference on machine learning},
  pages={19730--19742},
  year={2023},
  organization={PMLR}
}

@inproceedings{gao-etal-2024-multilingual,
    title = "Multilingual Pretraining and Instruction Tuning Improve Cross-Lingual Knowledge Alignment, But Only Shallowly",
    author = "Gao, Changjiang  and
      Hu, Hongda  and
      Hu, Peng  and
      Chen, Jiajun  and
      Li, Jixing  and
      Huang, Shujian",
    editor = "Duh, Kevin  and
      Gomez, Helena  and
      Bethard, Steven",
    booktitle = "Proceedings of the 2024 Conference of the North American Chapter of the Association for Computational Linguistics: Human Language Technologies (Volume 1: Long Papers)",
    month = jun,
    year = "2024",
    address = "Mexico City, Mexico",
    publisher = "Association for Computational Linguistics",
    url = "https://aclanthology.org/2024.naacl-long.339/",
    doi = "10.18653/v1/2024.naacl-long.339",
    pages = "6101--6117"
}

@inproceedings{hammerl2024understanding,
  title={Understanding cross-lingual Alignment—A survey},
  author={H{\"a}mmerl, Katharina and Libovick{\`y}, Jind{\v{r}}ich and Fraser, Alexander},
  booktitle={Findings of the Association for Computational Linguistics: ACL 2024},
  pages={10922--10943},
  year={2024}
}

@inproceedings{fang2025see,
  title={To see a world in a spark of neuron: Disentangling multi-task interference for training-free model merging},
  author={Fang, Zitao and Du, Guodong and Yu, Shuyang and Guo, Yifei and Zhang, Yiwei and Cao, Yiyao and Li, Jing and Tang, Ho-Kin and Goh, Sim Kuan},
  booktitle={Proceedings of the 2025 Conference on Empirical Methods in Natural Language Processing},
  pages={15731--15751},
  year={2025}
}

@article{ji2024emma,
  title={Emma-500: Enhancing massively multilingual adaptation of large language models},
  author={Ji, Shaoxiong and Li, Zihao and Paavola, Jaakko and Lin, Peiqin and Chen, Pinzhen and O'Brien, Dayy{\'a}n and Luo, Hengyu and Sch{\"u}tze, Hinrich and Tiedemann, J{\"o}rg and Haddow, Barry},
  journal={arXiv preprint arXiv:2409.17892},
  year={2024}
}

@inproceedings{rasheed2025palo,
  title={Palo: A polyglot large multimodal model for 5b people},
  author={Rasheed, Hanoona and Maaz, Muhammad and Shaker, Abdelrahman and Khan, Salman and Cholakkal, Hisham and Anwer, Rao M and Baldwin, Tim and Felsberg, Michael and Khan, Fahad S},
  booktitle={Proceedings of the Winter Conference on Applications of Computer Vision},
  pages={1745--1754},
  year={2025}
}

@article{singh2025openai,
  title={Openai gpt-5 system card},
  author={Singh, Aaditya and Fry, Adam and Perelman, Adam and Tart, Adam and Ganesh, Adi and El-Kishky, Ahmed and McLaughlin, Aidan and Low, Aiden and Ostrow, AJ and Ananthram, Akhila and others},
  journal={arXiv preprint arXiv:2601.03267},
  year={2025}
}

@article{wang2025internvl3,
  title={Internvl3. 5: Advancing open-source multimodal models in versatility, reasoning, and efficiency},
  author={Wang, Weiyun and Gao, Zhangwei and Gu, Lixin and Pu, Hengjun and Cui, Long and Wei, Xingguang and Liu, Zhaoyang and Jing, Linglin and Ye, Shenglong and Shao, Jie and others},
  journal={arXiv preprint arXiv:2508.18265},
  year={2025}
}

@article{team2023gemini,
  title={Gemini: a family of highly capable multimodal models},
  author={Team, Gemini and Anil, Rohan and Borgeaud, Sebastian and Alayrac, Jean-Baptiste and Yu, Jiahui and Soricut, Radu and Schalkwyk, Johan and Dai, Andrew M and Hauth, Anja and Millican, Katie and others},
  journal={arXiv preprint arXiv:2312.11805},
  year={2023}
}

@inproceedings{xue2021mt5,
  title={mT5: A massively multilingual pre-trained text-to-text transformer},
  author={Xue, Linting and Constant, Noah and Roberts, Adam and Kale, Mihir and Al-Rfou, Rami and Siddhant, Aditya and Barua, Aditya and Raffel, Colin},
  booktitle={Proceedings of the 2021 conference of the North American chapter of the association for computational linguistics: Human language technologies},
  pages={483--498},
  year={2021}
}

@article{liu2023visual,
  title={Visual instruction tuning},
  author={Liu, Haotian and Li, Chunyuan and Wu, Qingyang and Lee, Yong Jae},
  journal={Advances in neural information processing systems},
  volume={36},
  pages={34892--34916},
  year={2023}
}

@article{touvron2023llama,
  title={Llama: Open and efficient foundation language models},
  author={Touvron, Hugo and Lavril, Thibaut and Izacard, Gautier and Martinet, Xavier and Lachaux, Marie-Anne and Lacroix, Timoth{\'e}e and Rozi{\`e}re, Baptiste and Goyal, Naman and Hambro, Eric and Azhar, Faisal and others},
  journal={arXiv preprint arXiv:2302.13971},
  year={2023}
}

@inproceedings{zhu2024model,
  title={Model tailor: mitigating catastrophic forgetting in multi-modal large language models},
  author={Zhu, Didi and Sun, Zhongyi and Li, Zexi and Shen, Tao and Yan, Ke and Ding, Shouhong and Wu, Chao and Kuang, Kun},
  booktitle={Proceedings of the 41st International Conference on Machine Learning},
  pages={62581--62598},
  year={2024}
}

@inproceedings{ponti2020xcopa,
  title={XCOPA: A multilingual dataset for causal commonsense reasoning},
  author={Ponti, Edoardo Maria and Glava{\v{s}}, Goran and Majewska, Olga and Liu, Qianchu and Vuli{\'c}, Ivan and Korhonen, Anna},
  booktitle={Proceedings of the 2020 Conference on Empirical Methods in Natural Language Processing (EMNLP)},
  pages={2362--2376},
  year={2020}
}

@inproceedings{lin2022few,
  title={Few-shot learning with multilingual generative language models},
  author={Lin, Xi Victoria and Mihaylov, Todor and Artetxe, Mikel and Wang, Tianlu and Chen, Shuohui and Simig, Daniel and Ott, Myle and Goyal, Naman and Bhosale, Shruti and Du, Jingfei and others},
  booktitle={Proceedings of the 2022 conference on empirical methods in natural language processing},
  pages={9019--9052},
  year={2022}
}

@inproceedings{conneau2018xnli,
  title={XNLI: Evaluating cross-lingual sentence representations},
  author={Conneau, Alexis and Rinott, Ruty and Lample, Guillaume and Williams, Adina and Bowman, Samuel and Schwenk, Holger and Stoyanov, Veselin},
  booktitle={Proceedings of the 2018 conference on empirical methods in natural language processing},
  pages={2475--2485},
  year={2018}
}

@inproceedings{schneider2024m5,
  title={M5--a diverse benchmark to assess the performance of large multimodal models across multilingual and multicultural vision-language tasks},
  author={Schneider, Florian and Sitaram, Sunayana},
  booktitle={Findings of the Association for Computational Linguistics: EMNLP 2024},
  pages={4309--4345},
  year={2024}
}

@inproceedings{changpinyo2023maxm,
  title={Maxm: Towards multilingual visual question answering},
  author={Changpinyo, Soravit and Xue, Linting and Yarom, Michal and Thapliyal, Ashish and Szpektor, Idan and Amelot, Julien and Chen, Xi and Soricut, Radu},
  booktitle={Findings of the Association for Computational Linguistics: EMNLP 2023},
  pages={2667--2682},
  year={2023}
}

@inproceedings{liu2021visually,
  title={Visually grounded reasoning across languages and cultures},
  author={Liu, Fangyu and Bugliarello, Emanuele and Ponti, Edoardo Maria and Reddy, Siva and Collier, Nigel and Elliott, Desmond},
  booktitle={Proceedings of the 2021 Conference on Empirical Methods in Natural Language Processing},
  pages={10467--10485},
  year={2021}
}

@inproceedings{pfeiffer2022xgqa,
  title={xGQA: Cross-lingual visual question answering},
  author={Pfeiffer, Jonas and Geigle, Gregor and Kamath, Aishwarya and Steitz, Jan-Martin O and Roth, Stefan and Vuli{\'c}, Ivan and Gurevych, Iryna},
  booktitle={Findings of the association for computational linguistics: ACL 2022},
  pages={2497--2511},
  year={2022}
}

@article{yang2025qwen3,
  title={Qwen3 technical report},
  author={Yang, An and Li, Anfeng and Yang, Baosong and Zhang, Beichen and Hui, Binyuan and Zheng, Bo and Yu, Bowen and Gao, Chang and Huang, Chengen and Lv, Chenxu and others},
  journal={arXiv preprint arXiv:2505.09388},
  year={2025}
}

@article{yu2026afriquellm,
  title={AfriqueLLM: How Data Mixing and Model Architecture Impact Continued Pre-training for African Languages},
  author={Yu, Hao and Xu, Tianyi and Hedderich, Michael A and Hamidouche, Wassim and Zamir, Syed Waqas and Adelani, David Ifeoluwa},
  journal={arXiv preprint arXiv:2601.06395},
  year={2026}
}

@article{qwen2,
  title={Qwen2 Technical Report},
  year={2024}
}

@article{bai2025qwen3,
  title={Qwen3-vl technical report},
  author={Bai, Shuai and Cai, Yuxuan and Chen, Ruizhe and Chen, Keqin and Chen, Xionghui and Cheng, Zesen and Deng, Lianghao and Ding, Wei and Gao, Chang and Ge, Chunjiang and others},
  journal={arXiv preprint arXiv:2511.21631},
  year={2025}
}

@article{chen2024we,
  title={Are we on the right way for evaluating large vision-language models?},
  author={Chen, Lin and Li, Jinsong and Dong, Xiaoyi and Zhang, Pan and Zang, Yuhang and Chen, Zehui and Duan, Haodong and Wang, Jiaqi and Qiao, Yu and Lin, Dahua and others},
  journal={Advances in Neural Information Processing Systems},
  volume={37},
  pages={27056--27087},
  year={2024}
}

@inproceedings{yue2024mmmu,
  title={Mmmu: A massive multi-discipline multimodal understanding and reasoning benchmark for expert agi},
  author={Yue, Xiang and Ni, Yuansheng and Zhang, Kai and Zheng, Tianyu and Liu, Ruoqi and Zhang, Ge and Stevens, Samuel and Jiang, Dongfu and Ren, Weiming and Sun, Yuxuan and others},
  booktitle={Proceedings of the IEEE/CVF conference on computer vision and pattern recognition},
  pages={9556--9567},
  year={2024}
}

@article{li2024seed,
  title={Seed-bench-2-plus: Benchmarking multimodal large language models with text-rich visual comprehension},
  author={Li, Bohao and Ge, Yuying and Chen, Yi and Ge, Yixiao and Zhang, Ruimao and Shan, Ying},
  journal={arXiv preprint arXiv:2404.16790},
  year={2024}
}

@inproceedings{mogrovejo2024cvqa,
  title={CVQA: Culturally-diverse multilingual visual question answering benchmark},
  author={Mogrovejo, David Orlando Romero and Lyu, Chenyang and Wibowo, Haryo Akbarianto and G{\'o}ngora, Santiago and Mandal, Aishik and Purkayastha, Sukannya and Ortiz-Barajas, Jesus-German and Cueva, Emilio Villa and Baek, Jinheon and Jeong, Soyeong and others},
  booktitle={The Thirty-eight Conference on Neural Information Processing Systems Datasets and Benchmarks Track},
  year={2024}
}

@article{biderman2024lessons,
  title={Lessons from the trenches on reproducible evaluation of language models},
  author={Biderman, Stella and Schoelkopf, Hailey and Sutawika, Lintang and Gao, Leo and Tow, Jonathan and Abbasi, Baber and Aji, Alham Fikri and Ammanamanchi, Pawan Sasanka and Black, Sidney and Clive, Jordan and others},
  journal={arXiv preprint arXiv:2405.14782},
  year={2024}
}

@inproceedings{zhang2025lmms,
  title={Lmms-eval: Reality check on the evaluation of large multimodal models},
  author={Zhang, Kaichen and Li, Bo and Zhang, Peiyuan and Pu, Fanyi and Cahyono, Joshua Adrian and Hu, Kairui and Liu, Shuai and Zhang, Yuanhan and Yang, Jingkang and Li, Chunyuan and others},
  booktitle={Findings of the Association for Computational Linguistics: NAACL 2025},
  pages={881--916},
  year={2025}
}

@article{wei2025optmerge,
  title={OptMerge: Unifying Multimodal LLM Capabilities and Modalities via Model Merging},
  author={Wei, Yongxian and Cheng, Runxi and Jin, Weike and Yang, Enneng and Shen, Li and Hou, Lu and Du, Sinan and Yuan, Chun and Cao, Xiaochun and Tao, Dacheng},
  journal={arXiv preprint arXiv:2505.19892},
  year={2025}
}

@inproceedings{bu2025alignx,
  title={AlignX: Advancing Multilingual Large Language Models with Multilingual Representation Alignment},
  author={Bu, Mengyu and Zhang, Shaolei and He, Zhongjun and Wu, Hua and Feng, Yang},
  booktitle={Proceedings of the 2025 Conference on Empirical Methods in Natural Language Processing},
  pages={6471--6500},
  year={2025}
}

@article{zhao2024large,
  title={How do large language models handle multilingualism?},
  author={Zhao, Yiran and Zhang, Wenxuan and Chen, Guizhen and Kawaguchi, Kenji and Bing, Lidong},
  journal={Advances in Neural Information Processing Systems},
  volume={37},
  pages={15296--15319},
  year={2024}
}

@article{chen2026understanding,
  title={Understanding Multilingualism in Mixture-of-Experts LLMs: Routing Mechanism, Expert Specialization, and Layerwise Steering},
  author={Chen, Yuxin and Cai, Zhengzhou and Ji, Xiangtian and Zhao, Weixiang and Zhang, An and Wang, Xiang and Chua, Tat-Seng},
  journal={arXiv preprint arXiv:2601.14050},
  year={2026}
}

@inproceedings{huang2025neurons,
  title={From Neurons to Semantics: Evaluating Cross-Linguistic Alignment Capabilities of Large Language Models via Neurons Alignment},
  author={Huang, Chongxuan and Ye, Yongshi and Fu, Biao and Su, Qifeng and Shi, Xiaodong},
  booktitle={Proceedings of the 63rd Annual Meeting of the Association for Computational Linguistics (Volume 1: Long Papers)},
  pages={28956--28974},
  year={2025}
}

@article{bu2026language,
  title={Language on Demand, Knowledge at Core: Composing LLMs with Encoder-Decoder Translation Models for Extensible Multilinguality},
  author={Bu, Mengyu and Feng, Yang},
  journal={arXiv preprint arXiv:2603.17512},
  year={2026}
}

@inproceedings{chang2022geometry,
  title={The geometry of multilingual language model representations},
  author={Chang, Tyler and Tu, Zhuowen and Bergen, Benjamin},
  booktitle={Proceedings of the 2022 Conference on Empirical Methods in Natural Language Processing},
  pages={119--136},
  year={2022}
}

@inproceedings{lopo2025language,
  title={Language Surgery in Multilingual Large Language Models},
  author={Lopo, Joanito Agili and Habibi, Muhammad Ravi Shulthan and Wong, Tack Hwa and Ghozali, Muhammad Ilham and Koto, Fajri and Winata, Genta Indra and Limkonchotiwat, Peerat and Aji, Alham Fikri and Cahyawijaya, Samuel},
  booktitle={Proceedings of the 5th Workshop on Multilingual Representation Learning (MRL 2025)},
  pages={438--467},
  year={2025}
}

@inproceedings{wendler2024llamas,
  title={Do llamas work in english? on the latent language of multilingual transformers},
  author={Wendler, Chris and Veselovsky, Veniamin and Monea, Giovanni and West, Robert},
  booktitle={Proceedings of the 62nd Annual Meeting of the Association for Computational Linguistics (Volume 1: Long Papers)},
  pages={15366--15394},
  year={2024}
}

@inproceedings{pikabea2025breaking,
  title={Breaking language barriers in visual language models via multilingual textual regularization},
  author={Pikabea, I{\~n}igo and Lacunza, I{\~n}aki and Velasco, Oriol Pareras and Escolano, Carlos and Gonzalez-Agirre, Aitor and Hernando, Javier and Villegas, Marta},
  booktitle={Proceedings of the 14th International Joint Conference on Natural Language Processing and the 4th Conference of the Asia-Pacific Chapter of the Association for Computational Linguistics},
  pages={299--337},
  year={2025}
}

@inproceedings{geigle2024mblip,
  title={mblip: Efficient bootstrapping of multilingual vision-llms},
  author={Geigle, Gregor and Jain, Abhay and Timofte, Radu and Glava{\v{s}}, Goran},
  booktitle={Proceedings of the 3rd Workshop on Advances in Language and Vision Research (ALVR)},
  pages={7--25},
  year={2024}
}

@inproceedings{sunparrot,
  title={Parrot: Multilingual Visual Instruction Tuning},
  author={Sun, Hai-Long and Zhou, Da-Wei and Li, Yang and Lu, Shiyin and Yi, Chao and Chen, Qing-Guo and Xu, Zhao and Luo, Weihua and Zhang, Kaifu and Zhan, De-Chuan and others},
  booktitle={Forty-second International Conference on Machine Learning}
}

@inproceedings{fan2025language,
  title={Language-Specific Layer Matters: Efficient Multilingual Enhancement for Large Vision-Language Models},
  author={Fan, Yuchun and Wang, Yilin and Mu, Yongyu and Huang, Lei and Li, Bei and Feng, Xiaocheng and Xiao, Tong and Zhu, Jingbo},
  booktitle={Findings of the Association for Computational Linguistics: EMNLP 2025},
  pages={12473--12500},
  year={2025}
}

@article{alayrac2022flamingo,
  title={Flamingo: a visual language model for few-shot learning},
  author={Alayrac, Jean-Baptiste and Donahue, Jeff and Luc, Pauline and Miech, Antoine and Barr, Iain and Hasson, Yana and Lenc, Karel and Mensch, Arthur and Millican, Katherine and Reynolds, Malcolm and others},
  journal={Advances in neural information processing systems},
  volume={35},
  pages={23716--23736},
  year={2022}
}

@inproceedings{geigle2025centurio,
  title={Centurio: On drivers of multilingual ability of large vision-language model},
  author={Geigle, Gregor and Schneider, Florian and Holtermann, Carolin and Biemann, Chris and Timofte, Radu and Lauscher, Anne and Glava{\v{s}}, Goran},
  booktitle={Proceedings of the 63rd Annual Meeting of the Association for Computational Linguistics (Volume 1: Long Papers)},
  pages={2831--2881},
  year={2025}
}

@article{wang2025more,
  title={Why do more experts fail? a theoretical analysis of model merging},
  author={Wang, Zijing and Xu, Xingle and Liu, Yongkang and Zhang, Yiqun and Lin, Peiqin and Feng, Shi and Yang, Xiaocui and Wang, Daling and Sch{\"u}tze, Hinrich},
  journal={arXiv preprint arXiv:2505.21226},
  year={2025}
}

@article{wang2025scaling,
  title={Scaling Intelligence Through Model Merging: A Comprehensive Survey},
  author={Wang, Zijing and Liu, Yongkang and Luo, Yingfeng and Wang, Ming and Song, Zhen and Feng, Shi and Yang, Xiaocui and Lin, Dingyang and Wang, Daling and Zhang, Yifei and others},
  journal={Authorea Preprints},
  year={2025},
  publisher={Authorea}
}

@article{yang2026model,
  title={Model merging in llms, mllms, and beyond: Methods, theories, applications, and opportunities},
  author={Yang, Enneng and Shen, Li and Guo, Guibing and Wang, Xingwei and Cao, Xiaochun and Zhang, Jie and Tao, Dacheng},
  journal={ACM Computing Surveys},
  volume={58},
  number={8},
  pages={1--41},
  year={2026},
  publisher={ACM New York, NY}
}

@inproceedings{marczak2024magmax,
  title={Magmax: Leveraging model merging for seamless continual learning},
  author={Marczak, Daniel and Twardowski, Bart{\l}omiej and Trzci{\'n}ski, Tomasz and Cygert, Sebastian},
  booktitle={European Conference on Computer Vision},
  pages={379--395},
  year={2024},
  organization={Springer}
}

@article{he2024localize,
  title={Localize-and-stitch: Efficient model merging via sparse task arithmetic},
  author={He, Yifei and Hu, Yuzheng and Lin, Yong and Zhang, Tong and Zhao, Han},
  journal={arXiv preprint arXiv:2408.13656},
  year={2024}
}

@inproceedings{chengwhoever,
  title={Whoever Started the interference Should End It: Guiding Data-Free Model Merging via Task Vectors},
  author={Cheng, Runxi and Xiong, Feng and Wei, Yongxian and Zhu, Wanyun and Yuan, Chun},
  booktitle={Forty-second International Conference on Machine Learning}
}

@inproceedings{du2025adamms,
  title={Adamms: Model merging for heterogeneous multimodal large language models with unsupervised coefficient optimization},
  author={Du, Yiyang and Wang, Xiaochen and Chen, Chi and Ye, Jiabo and Wang, Yiru and Li, Peng and Yan, Ming and Zhang, Ji and Huang, Fei and Sui, Zhifang and others},
  booktitle={Proceedings of the Computer Vision and Pattern Recognition Conference},
  pages={9413--9422},
  year={2025}
}

@article{wang2026plam,
  title={PlaM: Training-Free Plateau-Guided Model Merging for Better Visual Grounding in MLLMs},
  author={Wang, Zijing and Liu, Yongkang and Wang, Mingyang and Nie, Ercong and Chen, Deyuan and Zhao, Zhengjie and Feng, Shi and Wang, Daling and Yang, Xiaocui and Zhang, Yifei and others},
  journal={arXiv preprint arXiv:2601.07645},
  year={2026}
}

@article{zheng2025unveiling,
  title={Unveiling Intrinsic Text Bias in Multimodal Large Language Models through Attention Key-Space Analysis},
  author={Zheng, Xinhan and Wu, Huyu and Wang, Xueting and Jiang, Haiyun},
  journal={arXiv preprint arXiv:2510.26721},
  year={2025}
}

@misc{luo2025sizedoesfitall,
      title={One Size Does Not Fit All: A Distribution-Aware Sparsification for More Precise Model Merging}, 
      author={Yingfeng Luo and Dingyang Lin and Junxin Wang and Ziqiang Xu and Kaiyan Chang and Tong Zheng and Bei Li and Anxiang Ma and Tong Xiao and Zhengtao Yu and Jingbo Zhu},
      year={2025},
      eprint={2508.06163},
      archivePrefix={arXiv},
      primaryClass={cs.CL},
      url={https://arxiv.org/abs/2508.06163}, 
}

@inproceedings{nguyen-etal-2024-culturax,
    title = "{C}ultura{X}: A Cleaned, Enormous, and Multilingual Dataset for Large Language Models in 167 Languages",
    author = "Nguyen, Thuat  and
      Nguyen, Chien Van  and
      Lai, Viet Dac  and
      Man, Hieu  and
      Ngo, Nghia Trung  and
      Dernoncourt, Franck  and
      Rossi, Ryan A.  and
      Nguyen, Thien Huu",
    editor = "Calzolari, Nicoletta  and
      Kan, Min-Yen  and
      Hoste, Veronique  and
      Lenci, Alessandro  and
      Sakti, Sakriani  and
      Xue, Nianwen",
    booktitle = "Proceedings of the 2024 Joint International Conference on Computational Linguistics, Language Resources and Evaluation (LREC-COLING 2024)",
    month = may,
    year = "2024",
    address = "Torino, Italia",
    publisher = "ELRA and ICCL",
    url = "https://aclanthology.org/2024.lrec-main.377/",
    pages = "4226--4237"
}

@inproceedings{tao-etal-2024-unlocking,
    title = "Unlocking the Potential of Model Merging for Low-Resource Languages",
    author = "Tao, Mingxu  and
      Zhang, Chen  and
      Huang, Quzhe  and
      Ma, Tianyao  and
      Huang, Songfang  and
      Zhao, Dongyan  and
      Feng, Yansong",
    editor = "Al-Onaizan, Yaser  and
      Bansal, Mohit  and
      Chen, Yun-Nung",
    booktitle = "Findings of the Association for Computational Linguistics: EMNLP 2024",
    month = nov,
    year = "2024",
    address = "Miami, Florida, USA",
    publisher = "Association for Computational Linguistics",
    url = "https://aclanthology.org/2024.findings-emnlp.508/",
    doi = "10.18653/v1/2024.findings-emnlp.508",
    pages = "8705--8720"
}

\clearpage

\appendix

\section{Baselines and Benchmarks}
\label{app:exp}

All experiments were conducted on NVIDIA RTX A6000 GPUs with 48GB memory per GPU.

\subsection{Merging Methods}

\paragraph{Task Arithmetic.}
Task Arithmetic~\cite{ilharcoediting} represents a task as the parameter difference between a fine-tuned model and its base model, and composes capabilities through linear operations on these task vectors. It provides the basic formulation of training-free model editing and serves as the foundation for many subsequent merging methods.

\paragraph{DARE.}
DARE~\cite{yu2024language} sparsifies fine-tuning deltas by randomly dropping a large fraction of parameters and rescaling the remainder. This procedure is motivated by the redundancy of task vectors and is intended to reduce interference while largely preserving the effect of the original update.

\paragraph{TIES-Merging.}
TIES-Merging~\cite{yadav2023ties} trims small-magnitude updates, resolves sign disagreement across task vectors, and merges only the coordinates that are consistent with the elected sign. The method is designed to explicitly reduce destructive interference caused by conflicting residual directions.

\paragraph{Breadcrumbs.}
Breadcrumbs~\cite{davari2024model} constructs sparse masks over task vectors to retain informative updates while filtering out both negligible changes and large disruptive outliers. 

\paragraph{PCB-Merging.}
PCB-Merging~\cite{du2024parameter} formulates merging as a parameter competition balancing problem. It estimates parameter importance within each task model and competition across task models, then selectively retains and rescales updates to preserve task-relevant information while reducing harmful conflicts during composition.

\paragraph{STF.}
STF~\cite{qiu2025superpose} approaches model merging from a representation perspective by superposing task-specific features rather than relying solely on direct parameter arithmetic. It emphasizes that compatibility across tasks can be modeled through structured interactions in feature space, providing an alternative to purely parameter-level heuristics.

\paragraph{NeuroMerging.}
NeuroMerging~\cite{fang2025see} studies merging through neuronal structure by disentangling task-specific information into complementary neuronal subspaces. It aims to mitigate multi-task interference at the neuron level and thereby better preserve complementary knowledge across source models.

\subsection{Models}

\paragraph{Emma-500.}
Emma-500~\cite{ji2024emma} is a massively multilingual language model obtained by continual pre-training of LLaMA-2-7B on the MaLA corpus, with coverage extending to 546 languages. The model is designed to improve multilingual performance, especially for underrepresented languages, while preserving compatibility with the original LLaMA-2-7B backbone. In our experiments, Emma-500 serves as the multilingual source model that provides the multilingual residual to be merged into the multimodal anchor.

\paragraph{Palo.}
Palo~\cite{rasheed2025palo} is a multilingual multimodal model developed to support visual reasoning in 10 major languages covering roughly 5 billion people. It is built through multilingual multimodal instruction tuning, using a semi-automated translation pipeline to extend English multimodal supervision to multiple target languages, and is reported at several model scales including 7B. In our experiments, Palo is used as a trained multilingual multimodal reference model rather than a merging source, providing a comparison point for performance achievable through additional multilingual multimodal training.

\paragraph{PLAST-7B.}
PLAST~\cite{fan2025language} is an efficient multilingual enhancement method for large vision-language models. Starting from LLaVA-v1.5-7B, it identifies shallow language-specific layers through neuron-activation analysis and then fine-tunes only these layers using question-translation pairs, with the goal of improving multilingual understanding while updating only a small fraction of parameters.

\paragraph{AfriqueQwen-8B.}
AfriqueQwen-8B~\cite{yu2026afriquellm} is a multilingual language model obtained by continued pre-training Qwen3-8B on a 26B token mixture covering 20 African languages. It is part of the AfriqueLLM suite, which systematically studies how data composition and model architecture affect adaptation to African languages. In our experiments, AfriqueQwen-8B serves as the multilingual source model in the Qwen3-based shared-backbone setting.

\paragraph{Pangea-7B.}
Pangea-7B~\cite{yue2024pangea} is a multilingual multimodal model trained through multilingual multimodal instruction tuning on PangeaIns, a 6M-sample dataset spanning 39 languages. It is designed to support multilingual and cross-cultural multimodal reasoning, and is evaluated in the original paper on PangeaBench, a benchmark suite covering 14 datasets in 47 languages. In our experiments, Pangea-7B is used as an already trained multilingual multimodal checkpoint in the Qwen2-based shared-backbone setting to test whether the proposed method can further improve an existing multilingual multimodal model.

\subsection{Benchmarks}

\paragraph{XCOPA.}
XCOPA~\cite{ponti2020xcopa} is a multilingual benchmark for causal commonsense reasoning in 11 languages. Each example asks the model to choose the more plausible cause or effect for a given event, making it a compact test of cross-lingual commonsense understanding. 

\paragraph{XStoryCloze.}
XStoryCloze~\cite{lin2022few} is a multilingual story completion benchmark spanning 11 languages. The task requires selecting the more coherent ending for a short narrative context, probing discourse-level commonsense and narrative understanding across languages. 

\paragraph{XNLI.}
XNLI~\cite{conneau2018xnli} is a standard cross-lingual natural language inference benchmark covering 15 languages. It extends MultiNLI with translated development and test sets and is widely used to evaluate multilingual sentence understanding.

\paragraph{xMMMU.}
xMMMU~\cite{yue2024pangea} is the multilingual extension of MMMU introduced in PangeaBench, designed to evaluate multilingual multimodal understanding and reasoning across a broad range of subjects and languages.

\paragraph{MaXM.}
MaXM~\cite{changpinyo2023maxm} is a multilingual visual question answering benchmark covering seven languages. It is constructed through a scalable translation-based pipeline and serves as a test-only benchmark for multilingual VQA.

\paragraph{MaRVL.}
MaRVL~\cite{liu2021visually} is a multilingual and multicultural benchmark for visually grounded reasoning. It pairs images with statements written by native speakers and asks the model to judge whether each statement is true or false, with data designed to reflect linguistic and cultural diversity. 

\paragraph{xGQA.}
xGQA~\cite{pfeiffer2022xgqa} extends the GQA benchmark to seven typologically diverse languages for cross-lingual visual question answering. It is designed to reveal multilingual multimodal misalignment by measuring how well models transfer visual reasoning ability beyond English. This makes it particularly informative for testing whether multilingual transfer remains compatible with visual grounding.

\paragraph{CVQA.}
CVQA~\cite{mogrovejo2024cvqa} is a culturally diverse multilingual visual question answering benchmark that goes beyond translation-based multilingual evaluation. It includes culturally grounded images and questions collected across diverse countries, languages, and scripts, making it a useful testbed for both linguistic diversity and cultural coverage.

\paragraph{Afri-MCQA.}
Afri-MCQA~\cite{tonja2026afri} is a multilingual cultural multimodal question-answering benchmark covering about 7.5k image-grounded QA pairs across 15 African languages from 12 countries. The benchmark provides parallel English and native-language question-answer pairs in both text and speech, and is constructed by native speakers.

\paragraph{MMStar.}
MMStar~\cite{chen2024we} is a multimodal evaluation benchmark designed to test whether large vision-language models genuinely use visual information rather than exploiting shortcuts or benchmark artifacts. It emphasizes visually indispensable questions and more reliable evaluation design.

\paragraph{MMMU.}
MMMU~\cite{yue2024mmmu} is a large-scale benchmark for multimodal understanding and reasoning at a college-exam level of difficulty. It covers a wide range of disciplines, subjects, and image types, and is intended to measure advanced perception and domain-specific reasoning.

\paragraph{SEED-Bench-2-Plus.}
SEED-Bench-2-Plus~\cite{li2024seed} is a benchmark for text-rich visual comprehension in multimodal large language models. It focuses on scenarios where substantial textual content appears inside images, requiring models to jointly process visual layout and embedded text.

\section{Other Ablations}\label{app:ablations}
We further ablate the merge scope by applying DiM\textsuperscript{3} to different modules of the shared language model. `Embed\_Only' merges only the token embedding layer, `LLM\_Only' merges only the transformer backbone, and `Lmhead\_Only' merges only the output language modeling head. We also evaluate partial-layer variants that merge `embed\_tokens', a contiguous block of transformer layers, and `lm\_head', namely `0\_7layers', `8\_15layers', `16\_23layers', and `24\_31layers'. As shown in Table~\ref{tab:benchmark_results}, merging only a single module is insufficient, and partial-layer merging is consistently weaker than applying DiM\textsuperscript{3} to the full shared backbone. Among the partial variants, `24\_31layers' is the strongest and remains relatively close to the original LLaVA on general multimodal performance, suggesting that upper layers preserve a substantial portion of task-sensitive multimodal structure. Nevertheless, the full DiM\textsuperscript{3} merge yields the best overall results, which is consistent with the analysis in Section~\ref{sec:analysis}: the gain is not attributable to a single isolated module, but is associated with improved shared semantic alignment across the backbone, with the most pronounced changes occurring in intermediate layers while higher-layer task-sensitive structure is still preserved.

\begin{table*}[htbp]
\centering
\caption{Ablation on merge scope: results of applying DiM\textsuperscript{3} to different modules and layer ranges of the shared language model.}
\label{tab:benchmark_results}
\small
\renewcommand{\arraystretch}{0.95}
\setlength{\tabcolsep}{4pt}
\begin{adjustbox}{max width=\textwidth}
\begin{tabular}{lcccccc}
\toprule
\textbf{Variant} & \textbf{MMStar} & \textbf{MMMU} & \textbf{SEED} & \textbf{Afri-MCQA} & \textbf{MaXM} & \textbf{MaRVL} \\
\specialrule{0.08em}{0.15em}{0.1em}
DiM\textsuperscript{3} & \cellcolor{bestblue}\textbf{34.86} & \cellcolor{bestblue}\textbf{37.22} & 40.45 & \cellcolor{bestblue}\textbf{26.31} & \cellcolor{bestblue}\textbf{33.44} & \cellcolor{bestblue}\textbf{62.91} \\
Embed\_Only & 33.33 & \cellcolor{secondgreen}35.67 & \cellcolor{secondgreen}41.81 & 24.67 & 19.79 & 49.99 \\
LLM\_Only & 32.70 & 28.33 & 35.92 & 19.01 & 16.91 & \cellcolor{secondgreen}54.17 \\
Lmhead\_Only & 20.31 & 25.67 & 24.68 & 0.00 & 2.68 & 49.74 \\
0\_7layers & 10.40 & 27.44 & 12.08 & 1.01 & 2.66 & 49.72 \\
8\_15layers & 16.51 & 25.67 & 14.32 & 0.00 & 2.01 & 50.28 \\
16\_23layers & 28.11 & 30.33 & 35.70 & 0.00 & 1.82 & 50.40 \\
24\_31layers & \cellcolor{secondgreen}34.36 & 35.22 & \cellcolor{bestblue}\textbf{42.03} & \cellcolor{secondgreen}24.74 & \cellcolor{secondgreen}27.13 & 49.89 \\
\bottomrule
\end{tabular}
\end{adjustbox}
\end{table*}

\section{Limitations}\label{app:lim}

Our method currently relies on a reasonably compatible shared backbone between the multilingual source model and the multimodal anchor. When the source checkpoints differ substantially in architecture or tokenizer design, parameter-space composition may become less stable because the two residual updates no longer admit a clean common reference point. Extending training-free merging beyond this shared-backbone setting, especially to models with tokenizer or architectural mismatches, remains an important direction for future work.
\begin{table}[htbp]
\caption{Languages used in our experiments and their corresponding ISO 639-1/2 codes (57 in total).}
\label{tab:language_codes}
\centering

\begin{minipage}[t]{0.3\textwidth}
\centering
\begin{tabular}{>{\bfseries}c l}
\toprule
ISO 639 & Language \\
\midrule
Ak & Akan/Twi \\
Am & Amharic \\
Ar & Arabic \\
Bg & Bulgarian \\
Bn & Bengali \\
Br & Breton \\
De & German \\
El & Greek \\
En & English \\
Es & Spanish \\
Et & Estonian \\
Eu & Basque \\
Fil & Filipino \\
Fr & French \\
Ga & Irish \\
Ha & Hausa  \\
He & Hebrew \\
Hi & Hindi \\
Ht & Haitian \\
\bottomrule
\end{tabular}
\end{minipage}
\hfill
\begin{minipage}[t]{0.3\textwidth}
\centering
\begin{tabular}{>{\bfseries}c l}
\toprule
ISO 639 & Language \\
\midrule
Id & Indonesian \\
Ig & Igbo \\
It & Italian \\
Ja & Japanese \\
Jv & Javanese \\
Ki & Kikuyu   \\
Ko & Korean \\
Ln & Lingala \\
Lug & Luganda \\
Min & Minangkabau \\
Mn & Mongolian \\
Mr & Marathi \\
Ms & Malay \\
My & Burmese \\
No & Norwegian \\
Ny & Chichewa \\
Om & Oromo \\
Pt & Portuguese \\
Qu & Quechua \\
\bottomrule
\end{tabular}
\end{minipage}
\hfill
\begin{minipage}[t]{0.3\textwidth}
\centering
\begin{tabular}{>{\bfseries}c l}
\toprule
ISO 639 & Language \\
\midrule
Ro & Romanian \\
Ru & Russian \\
Rw & Kinyarwanda \\
Si & Sinhala \\
So & Somali \\
Sot & Sesotho \\
Su & Sundanese \\
Sw & Swahili \\
Ta & Tamil \\
Te & Telugu \\
Th & Thai \\
Ti & Tigrinya  \\
Tn & Setswana \\
Tr & Turkish \\
Ur & Urdu \\
Vi & Vietnamese \\
Yo & Yoruba \\
Zh & Chinese \\
Zu & Zulu \\
\bottomrule
\end{tabular}
\end{minipage}
\end{table}

\begin{table*}[t]
\centering
\caption{Results on XCOPA and XStoryCloze across multilingual settings. $\mathrm{Avg}_{\mathrm{mul}}$ denotes the average performance over all languages except English.}
\label{tab:xcopa_xstorycloze_methods}
\small
\renewcommand{\arraystretch}{0.95}
\setlength{\tabcolsep}{4pt}
\begin{adjustbox}{max width=\textwidth}
\begin{tabular}{lcccccccccccc}
\toprule
\multicolumn{13}{c}{\textbf{XCOPA}} \\
\specialrule{0.08em}{0.15em}{0.1em}
\textbf{Methods} & \textbf{Et} & \textbf{Ht} & \textbf{Id} & \textbf{It} & \textbf{Qu} & \textbf{Sw} & \textbf{Ta} & \textbf{Th} & \textbf{Tr} & \textbf{Vi} & \textbf{Zh} & $\mathbf{Avg}_{\mathbf{mul}}$ \\
\midrule
Task Arithmetic & 48.20 & 50.60 & 65.00 & 68.00 & 49.80 & 49.80 & 52.00 & 55.40 & 55.60 & 63.80 & 66.40 & 56.78 \\
DARE            & 48.60 & 50.60 & 63.60 & 67.20 & 50.80 & 50.60 & 52.80 & \cellcolor{bestblue}\textbf{57.40} & 56.00 & 60.20 & 64.20 & 56.55 \\
TIES-Merging    & 48.00 & 49.80 & \cellcolor{secondgreen}67.00 & \cellcolor{secondgreen}68.40 & 49.40 & 50.80 & 52.20 & 56.40 & 56.20 & \cellcolor{secondgreen}66.20 & \cellcolor{secondgreen}67.40 & \cellcolor{secondgreen}57.44 \\
Breadcrumbs     & 48.40 & 50.60 & 65.80 & 66.20 & 51.00 & 51.00 & 51.40 & \cellcolor{secondgreen}57.00 & 55.80 & 65.60 & 64.00 & 56.98 \\
PCB-Merging     & 48.00 & 52.80 & 60.40 & 60.40 & \cellcolor{secondgreen}51.20 & \cellcolor{secondgreen}56.60 & 55.80 & 56.20 & 55.00 & 56.60 & 58.00 & 55.55 \\
STF             & \cellcolor{secondgreen}53.80 & 55.00 & 59.80 & 53.40 & 49.80 & 52.80 & \cellcolor{secondgreen}56.60 & 53.00 & 57.40 & 59.80 & 58.00 & 55.40 \\
NeuroMerging    & 52.40 & \cellcolor{secondgreen}55.20 & \cellcolor{secondgreen}58.20 & 55.60 & 47.40 & 51.40 & 56.40 & 54.60 & \cellcolor{secondgreen}58.20 & 64.00 & 63.60 & 56.09 \\
DiM\textsuperscript{3}            & \cellcolor{bestblue}\textbf{60.60} & \cellcolor{bestblue}\textbf{55.80} & \cellcolor{bestblue}\textbf{73.40} & \cellcolor{bestblue}\textbf{70.80} & \cellcolor{bestblue}\textbf{55.00} & \cellcolor{bestblue}\textbf{61.40} & \cellcolor{bestblue}\textbf{60.80} & 56.60 & \cellcolor{bestblue}\textbf{62.00} & \cellcolor{bestblue}\textbf{70.00} & \cellcolor{bestblue}\textbf{70.80} & \cellcolor{bestblue}\textbf{63.38} \\
\specialrule{0.08em}{0.25em}{0.15em}
\multicolumn{13}{c}{\textbf{XStoryCloze}} \\
\specialrule{0.08em}{0.15em}{0.1em}
\textbf{Methods} & \textbf{Ar} & \textbf{Es} & \textbf{Eu} & \textbf{Hi} & \textbf{Id} & \textbf{My} & \textbf{Ru} & \textbf{Sw} & \textbf{Te} & \textbf{Zh} & \textbf{En} & $\mathbf{Avg}_{\mathbf{mul}}$ \\
\midrule
Task Arithmetic & 53.67 & \cellcolor{secondgreen}70.22 & 51.82 & 56.52 & 64.06 & 48.91 & 67.44 & 51.75 & 55.13 & 64.53 & 80.41 & 58.41 \\
DARE            & 51.62 & 67.70 & 52.22 & 53.34 & 61.42 & 48.91 & 64.92 & 50.50 & 54.20 & 63.27 & 78.95 & 56.81 \\
TIES-Merging    & 54.07 & \cellcolor{bestblue}\textbf{71.08} & 51.56 & 56.32 & \cellcolor{secondgreen}64.86 & 49.11 & \cellcolor{secondgreen}67.97 & 51.03 & 55.06 & \cellcolor{secondgreen}64.59 & \cellcolor{bestblue}\textbf{80.94} & \cellcolor{secondgreen}58.57 \\
Breadcrumbs     & 52.42 & 69.82 & 51.03 & 55.26 & 63.20 & 48.78 & 65.72 & 50.56 & 55.66 & 62.41 & \cellcolor{secondgreen}80.87 & 57.49 \\
PCB-Merging     & 53.21 & 62.41 & 52.08 & 53.34 & 53.34 & 43.94 & 60.42 & 49.50 & 52.08 & 61.48 & 67.31 & 54.18 \\
STF             & \cellcolor{secondgreen}60.23 & 55.06 & \cellcolor{secondgreen}53.54 & \cellcolor{secondgreen}58.90 & 56.12 & \cellcolor{secondgreen}50.03 & 57.31 & \cellcolor{secondgreen}53.54 & \cellcolor{secondgreen}58.44 & 59.36 & 61.15 & 56.25 \\
NeuroMerging    & 58.31 & 57.78 & 53.21 & 58.37 & 58.24 & 48.25 & 60.89 & 52.68 & 55.86 & 61.81 & 65.72 & 56.54 \\
DiM\textsuperscript{3}            & \cellcolor{bestblue}\textbf{64.06} & 69.82 & \cellcolor{bestblue}\textbf{63.80} & \cellcolor{bestblue}\textbf{64.92} & \cellcolor{bestblue}\textbf{68.96} & \cellcolor{bestblue}\textbf{55.20} & \cellcolor{bestblue}\textbf{70.75} & \cellcolor{bestblue}\textbf{62.14} & \cellcolor{bestblue}\textbf{62.67} & \cellcolor{bestblue}\textbf{65.92} & 79.02 & \cellcolor{bestblue}\textbf{64.82} \\
\bottomrule
\end{tabular}
\end{adjustbox}
\end{table*}

\begin{table*}[t]
\centering
\caption{Results on XNLI across 15 languages.}
\label{tab:xnli_methods}
\small
\renewcommand{\arraystretch}{0.95}
\setlength{\tabcolsep}{4pt}
\begin{adjustbox}{max width=\textwidth}
\begin{tabular}{lcccccccccccccccc}
\toprule
\textbf{Methods} & \textbf{Ar} & \textbf{Bg} & \textbf{De} & \textbf{El} & \textbf{Es} & \textbf{Fr} & \textbf{Hi} & \textbf{Ru} & \textbf{Sw} & \textbf{Th} & \textbf{Tr} & \textbf{Ur} & \textbf{Vi} & \textbf{Zh} & \textbf{En} & $\mathbf{Avg}_{\mathbf{mul}}$ \\
\specialrule{0.08em}{0.15em}{0.1em}
Task Arithmetic & 33.53 & 42.01 & 46.10 & 38.76 & 43.78 & 47.71 & 39.56 & 40.32 & 36.31 & \cellcolor{secondgreen}36.02 & 41.00 & 34.02 & \cellcolor{secondgreen}43.41 & \cellcolor{bestblue}\textbf{38.51} & 56.91 & 40.07 \\
DARE            & 33.69 & 43.29 & 42.49 & 37.91 & \cellcolor{secondgreen}45.22 & 45.42 & 38.55 & 40.76 & 34.82 & 34.46 & 38.07 & 33.57 & 41.73 & 37.91 & 56.63 & 39.14 \\
TIES-Merging    & 33.49 & 42.49 & 47.63 & 39.72 & \cellcolor{bestblue}\textbf{45.82} & 48.88 & 41.33 & 41.24 & 36.47 & 35.82 & 41.29 & 34.50 & 41.16 & 36.79 & \cellcolor{secondgreen}58.51 & \cellcolor{secondgreen}40.47 \\
Breadcrumbs     & \cellcolor{secondgreen}33.78 & \cellcolor{secondgreen}43.61 & \cellcolor{bestblue}\textbf{47.75} & 39.28 & \cellcolor{secondgreen}45.22 & \cellcolor{secondgreen}49.08 & 40.28 & \cellcolor{secondgreen}43.09 & 36.22 & 35.74 & 40.80 & 33.53 & 39.20 & \cellcolor{secondgreen}38.27 & \cellcolor{bestblue}\textbf{58.67} & 40.42 \\
PCB-Merging     & 33.13 & 41.73 & 44.02 & 33.41 & 41.12 & 45.10 & 33.33 & 40.44 & \cellcolor{secondgreen}37.35 & 33.49 & 34.38 & 33.45 & 35.90 & 33.17 & 40.24 & 37.14 \\
STF             & 33.49 & \cellcolor{bestblue}\textbf{46.10} & 42.77 & 40.24 & 40.84 & \cellcolor{bestblue}\textbf{51.41} & 38.07 & 36.71 & 33.33 & 33.65 & \cellcolor{bestblue}\textbf{43.57} & 38.31 & 37.51 & 31.93 & 38.55 & 39.14 \\
NeuroMerging    & 33.61 & 41.29 & \cellcolor{secondgreen}44.10 & 41.04 & 43.57 & 47.71 & \cellcolor{secondgreen}42.81 & 39.60 & 34.02 & 35.50 & 40.84 & \cellcolor{secondgreen}43.29 & 38.27 & 34.46 & 41.04 & 40.01 \\
DiM\textsuperscript{3}            & \cellcolor{bestblue}\textbf{35.10} & 43.25 & \cellcolor{secondgreen}47.71 & \cellcolor{bestblue}\textbf{43.17} & 44.18 & 45.82 & \cellcolor{bestblue}\textbf{45.94} & \cellcolor{bestblue}\textbf{44.02} & \cellcolor{bestblue}\textbf{43.57} & \cellcolor{bestblue}\textbf{39.24} & \cellcolor{secondgreen}42.37 & \cellcolor{bestblue}\textbf{43.90} & \cellcolor{bestblue}\textbf{44.62} & 38.11 & 50.68 & \cellcolor{bestblue}\textbf{42.93} \\
\bottomrule
\end{tabular}
\end{adjustbox}
\end{table*}

\begin{table*}[t]
\centering
\caption{Results on XCOPA and XStoryCloze across multilingual settings.}
\label{tab:xcopa_xstorycloze_models}
\small
\renewcommand{\arraystretch}{0.95}
\setlength{\tabcolsep}{4pt}
\begin{adjustbox}{max width=\textwidth}
\begin{tabular}{lcccccccccccc}
\toprule
\multicolumn{13}{c}{\textbf{XCOPA}} \\
\specialrule{0.08em}{0.15em}{0.1em}
\textbf{Models} & \textbf{Et} & \textbf{Ht} & \textbf{Id} & \textbf{It} & \textbf{Qu} & \textbf{Sw} & \textbf{Ta} & \textbf{Th} & \textbf{Tr} & \textbf{Vi} & \textbf{Zh} & $\mathbf{Avg}_{\mathbf{mul}}$ \\
\midrule
Llama\_2\_7B   & 49.00 & 50.40 & 62.00 & 62.80 & 50.80 & 52.20 & 53.60 & \cellcolor{secondgreen}{56.40} & 54.80 & 63.40 & 63.60 & 56.27 \\
EMMA-500       & \cellcolor{bestblue}\textbf{61.40} & \cellcolor{bestblue}\textbf{57.00} & \cellcolor{bestblue}\textbf{73.60} & \cellcolor{secondgreen}{65.40} & 50.00 & \cellcolor{bestblue}\textbf{64.80} & \cellcolor{secondgreen}{60.60} & 56.20 & \cellcolor{secondgreen}{61.00} & \cellcolor{bestblue}\textbf{70.20} & \cellcolor{secondgreen}{67.00} & \cellcolor{secondgreen}{62.47} \\
LLaVA-v1.5-7B  & 49.40 & 51.60 & 50.80 & 50.00 & \cellcolor{secondgreen}{51.20} & 52.00 & 53.80 & 54.60 & 54.20 & 51.00 & 53.20 & 51.98 \\
DiM\textsuperscript{3}           & \cellcolor{secondgreen}{60.60} & \cellcolor{secondgreen}{55.80} & \cellcolor{secondgreen}{73.40} & \cellcolor{bestblue}\textbf{70.80} & \cellcolor{bestblue}\textbf{55.00} & \cellcolor{secondgreen}{61.40} & \cellcolor{bestblue}\textbf{60.80} & \cellcolor{bestblue}\textbf{56.60} & \cellcolor{bestblue}\textbf{62.00} & \cellcolor{secondgreen}{70.00} & \cellcolor{bestblue}\textbf{70.80} & \cellcolor{bestblue}\textbf{63.38} \\
\specialrule{0.08em}{0.25em}{0.15em}
\multicolumn{13}{c}{\textbf{XStoryCloze}} \\
\specialrule{0.08em}{0.15em}{0.1em}
\textbf{Models} & \textbf{Ar} & \textbf{Es} & \textbf{Eu} & \textbf{Hi} & \textbf{Id} & \textbf{My} & \textbf{Ru} & \textbf{Sw} & \textbf{Te} & \textbf{Zh} & \textbf{En} & $\mathbf{Avg}_{\mathbf{mul}}$ \\
\midrule
Llama\_2\_7B   & 50.17 & 67.70 & 50.56 & 53.54 & 59.36 & 48.51 & 63.34 & 50.43 & 54.33 & 59.50 & \cellcolor{secondgreen}{77.10} & 55.74 \\
EMMA-500       & \cellcolor{bestblue}\textbf{65.25} & \cellcolor{bestblue}\textbf{69.89} & \cellcolor{bestblue}\textbf{64.39} & \cellcolor{secondgreen}{63.80} & \cellcolor{bestblue}\textbf{69.03} & \cellcolor{bestblue}\textbf{58.44} & \cellcolor{secondgreen}{67.84} & \cellcolor{bestblue}\textbf{63.34} & \cellcolor{bestblue}\textbf{64.00} & \cellcolor{secondgreen}{61.75} & 75.65 & \cellcolor{secondgreen}{64.77} \\
LLaVA-v1.5-7B  & 47.98 & 51.89 & 50.23 & 47.58 & 48.64 & 47.32 & 50.17 & 48.58 & 53.67 & 48.78 & 66.71 & 49.48 \\
DiM\textsuperscript{3}           & \cellcolor{secondgreen}{64.06} & \cellcolor{secondgreen}{69.82} & \cellcolor{secondgreen}{63.80} & \cellcolor{bestblue}\textbf{64.92} & \cellcolor{secondgreen}{68.96} & \cellcolor{secondgreen}{55.20} & \cellcolor{bestblue}\textbf{70.75} & \cellcolor{secondgreen}{62.14} & \cellcolor{secondgreen}{62.67} & \cellcolor{bestblue}\textbf{65.92} & \cellcolor{bestblue}\textbf{79.02} & \cellcolor{bestblue}\textbf{64.82} \\
\bottomrule
\end{tabular}
\end{adjustbox}
\vspace{-0.8em}
\end{table*}

\begin{table*}[t]
\centering
\caption{Results on XNLI across multilingual settings.}
\label{tab:xnli_models}
\small
\renewcommand{\arraystretch}{0.95}
\setlength{\tabcolsep}{4pt}
\begin{adjustbox}{max width=\textwidth}
\begin{tabular}{lcccccccccccccccc}
\toprule
\textbf{Models} & \textbf{Ar} & \textbf{Bg} & \textbf{De} & \textbf{El} & \textbf{Es} & \textbf{Fr} & \textbf{Hi} & \textbf{Ru} & \textbf{Sw} & \textbf{Th} & \textbf{Tr} & \textbf{Ur} & \textbf{Vi} & \textbf{Zh} & \textbf{En} & $\mathbf{Avg}_{\mathbf{mul}}$ \\
\specialrule{0.08em}{0.15em}{0.1em}
Llama-2-7B   & 34.22 & 42.13 & \cellcolor{bestblue}\textbf{48.11} & 36.59 & 42.45 & \cellcolor{bestblue}\textbf{50.28} & 38.71 & 39.88 & 35.02 & 36.27 & 36.87 & 33.90 & 38.07 & 34.46 & \cellcolor{bestblue}\textbf{55.66} & 39.07 \\
EMMA-500     & \cellcolor{secondgreen}{34.78} & \cellcolor{bestblue}\textbf{46.27} & 47.07 & \cellcolor{bestblue}\textbf{45.86} & \cellcolor{bestblue}\textbf{46.75} & 34.78 & \cellcolor{bestblue}\textbf{47.59} & \cellcolor{bestblue}\textbf{46.55} & \cellcolor{bestblue}\textbf{46.18} & \cellcolor{bestblue}\textbf{41.97} & \cellcolor{bestblue}\textbf{44.86} & \cellcolor{bestblue}\textbf{45.98} & \cellcolor{bestblue}\textbf{47.03} & 35.22 & \cellcolor{secondgreen}{53.78} & \cellcolor{bestblue}\textbf{43.64} \\
LLaVA-v1.5-7B& 33.57 & 33.45 & 38.71 & 36.22 & 37.55 & 37.87 & 34.98 & 35.30 & 33.65 & 34.06 & 35.82 & 34.58 & 36.06 & \cellcolor{secondgreen}{35.86} & 52.81 & 35.55 \\
DiM\textsuperscript{3}         & \cellcolor{bestblue}\textbf{35.10} & \cellcolor{secondgreen}{43.25} & \cellcolor{secondgreen}{47.71} & \cellcolor{secondgreen}{43.17} & \cellcolor{secondgreen}{44.18} & \cellcolor{secondgreen}{45.82} & \cellcolor{secondgreen}{45.94} & \cellcolor{secondgreen}{44.02} & \cellcolor{secondgreen}{43.57} & \cellcolor{secondgreen}{39.24} & \cellcolor{secondgreen}{42.37} & \cellcolor{secondgreen}{43.90} & \cellcolor{secondgreen}{44.62} & \cellcolor{bestblue}\textbf{38.11} & 50.68 & \cellcolor{secondgreen}{42.93} \\
\bottomrule
\end{tabular}
\end{adjustbox}
\end{table*}

\begin{table*}[t]
\centering
\caption{Results on xMMMU and MaXM across multilingual settings.}
\label{tab:xmmmu_maxm_combined}
\small
\setlength{\tabcolsep}{3pt}
\begin{adjustbox}{max width=\textwidth}
\begin{tabular}{lcccccccc c cccccccc}
\toprule
& \multicolumn{8}{c}{\textbf{xMMMU}} & & \multicolumn{8}{c}{\textbf{MaXM}} \\
\cmidrule(lr){2-9}\cmidrule(lr){11-18}
Models 
& Ar & Fr & Hi & Id & Ja & Pt & En & $\mathrm{Avg}_{\mathrm{mul}}$
& 
& Fr & Hi & He & Ro & Th & Zh & En & $\mathrm{Avg}_{\mathrm{mul}}$ \\
\midrule
LLaVA-v1.5-7B
& 28.90 & 34.90 & 27.80 & 32.00 & 31.60 & 33.70 & \cellcolor{secondgreen}{35.70} & 31.48
&
& 32.20 & 16.92 & 12.50 & 15.14 & 16.04 & 27.80 & 49.81 & 20.10 \\
Task Arithmetic
& 30.20 & \cellcolor{secondgreen}{36.60} & 27.10 & 33.00 & 32.70 & \cellcolor{secondgreen}{36.40} & \cellcolor{bestblue}\textbf{35.90} & 32.67
&
& 33.33 & 21.15 & 12.50 & 16.90 & 18.66 & \cellcolor{secondgreen}{31.05} & 49.42 & 22.27 \\
DARE
& 28.90 & 34.90 & 27.50 & 31.60 & \cellcolor{secondgreen}{33.10} & 32.30 & 35.60 & 31.38
&
& 34.47 & 17.69 & 12.50 & 11.27 & 16.04 & 25.27 & 49.81 & 19.54 \\
TIES-Merging
& 28.50 & 32.20 & 26.80 & 33.00 & \cellcolor{bestblue}\textbf{35.70} & \cellcolor{bestblue}\textbf{37.40} & 34.70 & 32.27
&
& 33.71 & 26.92 & 21.79 & 16.90 & 17.54 & 28.52 & 48.25 & 24.23 \\
Breadcrumbs
& 27.20 & 32.60 & 25.80 & 30.30 & 30.90 & 32.70 & 31.90 & 29.92
&
& 31.44 & 24.62 & 19.64 & 19.01 & 10.82 & 25.99 & 48.64 & 21.92 \\
PCB-Merging
& 21.50 & 31.50 & 30.20 & 28.30 & 27.90 & 28.30 & 24.70 & 27.95
&
& 0.38 & 0.00 & 0.00 & 0.00 & 0.00 & 0.00 & 12.45 & 0.06 \\
STF
& \cellcolor{secondgreen}{31.50} & 34.20 & \cellcolor{secondgreen}{30.90} & \cellcolor{secondgreen}{34.00} & 29.00 & 32.70 & 31.00 & 32.05
&
& 10.61 & 10.00 & 11.79 & 7.75 & 11.94 & 14.44 & 18.29 & 11.09 \\
NeuroMerging
& 28.50 & \cellcolor{bestblue}\textbf{37.20} & 30.60 & \cellcolor{bestblue}\textbf{36.70} & \cellcolor{secondgreen}{33.10} & 34.70 & 33.80 & \cellcolor{secondgreen}{33.47}
&
& 25.00 & 14.23 & \cellcolor{secondgreen}{26.79} & 9.15 & \cellcolor{secondgreen}{22.76} & 25.27 & 30.74 & 20.53 \\
Palo-7B
& 30.90 & 33.20 & 28.90 & \cellcolor{secondgreen}{34.00} & 27.10 & 33.70 & 31.70 & 31.30
&
& 33.71 & 16.54 & 12.14 & 11.27 & 14.55 & 10.47 & \cellcolor{bestblue}\textbf{51.36} & 16.45 \\
Palo-13B
& 29.50 & 33.90 & \cellcolor{bestblue}\textbf{32.00} & 32.30 & 26.80 & 33.30 & 33.80 & 31.30
&
& \cellcolor{bestblue}\textbf{40.91} & \cellcolor{bestblue}\textbf{39.62} & 12.86 & \cellcolor{secondgreen}{33.45} & 14.55 & 30.69 & \cellcolor{bestblue}\textbf{51.36} & \cellcolor{secondgreen}{28.68} \\
PLAST-7B
& - & - & - & - & - & - & - & -
&
& - & 10.70 & 10.20 & \cellcolor{bestblue}\textbf{36.70} & 20.8 & - & - & - \\
DiM\textsuperscript{3}-7B
& \cellcolor{bestblue}\textbf{34.60} & \cellcolor{secondgreen}{36.60} & 29.20 & 32.30 & \cellcolor{bestblue}\textbf{35.70} & 34.70 & 32.60 & \cellcolor{bestblue}\textbf{33.85}
&
& \cellcolor{secondgreen}{35.98} & \cellcolor{secondgreen}{35.77} & \cellcolor{bestblue}\textbf{40.71} & 14.79 & \cellcolor{bestblue}\textbf{36.19} & \cellcolor{bestblue}\textbf{37.18} & \cellcolor{secondgreen}{50.58} & \cellcolor{bestblue}\textbf{33.44} \\
\bottomrule
\end{tabular}
\end{adjustbox}
\end{table*}

\begin{table*}[t]
\centering
\caption{Results on MaRVL and XGQA across multilingual settings.}
\label{tab:marvl_xgqa}
\small
\setlength{\tabcolsep}{3pt}
\begin{adjustbox}{max width=\textwidth}
\begin{tabular}{lccccccc c ccccccccc}
\toprule
& \multicolumn{7}{c}{\textbf{MaRVL}} & & \multicolumn{9}{c}{\textbf{XGQA}} \\
\cmidrule(lr){2-8} \cmidrule(lr){10-18}
Models
& Id & Sw & Ta & Tr & Zh & En & $\mathrm{Avg}_{\mathrm{mul}}$
&
& Bn & De & Id & Ko & Pt & Ru & Zh & En & $\mathrm{Avg}_{\mathrm{mul}}$ \\
\midrule
LLaVA-v1.5-7B
& 50.09 & 49.37 & 49.36 & 50.00 & 50.10 & 51.22 & 49.78
&
& 15.48 & 28.50 & 33.17 & 38.00 & 27.53 & 32.97 & 38.67 & \cellcolor{bestblue}\textbf{61.89} & 30.62 \\
Task Arithmetic
& 50.00 & 49.46 & 49.60 & 50.08 & 49.90 & 51.12 & 49.81
&
& 16.88 & 38.03 & 37.33 & \cellcolor{secondgreen}{41.93} & 34.42 & \cellcolor{secondgreen}{42.01} & 45.20 & \cellcolor{secondgreen}{61.50} & 36.54 \\
DARE
& 50.00 & 49.37 & 49.60 & 49.83 & 50.00 & 50.92 & 49.76
&
& 14.78 & 32.96 & 34.23 & 34.07 & 30.85 & 36.47 & 37.19 & 60.16 & 31.51 \\
TIES-Merging
& 50.00 & 49.37 & 49.28 & 50.00 & 49.90 & 50.99 & 49.71
&
& 18.54 & 38.21 & 38.24 & 40.59 & 35.71 & 41.85 & 43.32 & 58.37 & 36.64 \\
Breadcrumbs
& 50.00 & 49.37 & 49.52 & 50.00 & 49.31 & 51.13     & 49.64
&
& 15.66 & \cellcolor{secondgreen}{41.30} & \cellcolor{secondgreen}{38.26} & 37.01 & \cellcolor{secondgreen}{38.39} & 40.80 & 40.78 & 53.27 & 36.03 \\
STF
& 51.33 & \cellcolor{secondgreen}{50.27} & \cellcolor{secondgreen}{51.05} & 50.59 & \cellcolor{secondgreen}{51.09} & 52.49 & 50.87
&
& 0.91 & 16.96 & 11.11 & 9.42 & 10.40 & 11.24 & 16.06 & 20.56 & 10.87 \\
NeuroMerging
& 49.91 & 49.28 & 49.84 & 50.17 & 50.00 & 50.93 & 49.84
&
& 2.44 & 28.56 & 23.34 & 17.60 & 24.56 & 26.62 & 23.74 & 36.19 & 20.98 \\
Palo-7B
& \cellcolor{secondgreen}{51.95} & 49.55 & 49.44 & \cellcolor{secondgreen}{52.29} & 50.40 & \cellcolor{bestblue}\textbf{55.39} & \cellcolor{secondgreen}{50.73}
&
& \cellcolor{bestblue}\textbf{41.92} & 38.97 & 36.67 & 41.28 & 31.86 & 21.84 & \cellcolor{secondgreen}{45.94} & 60.25 & \cellcolor{secondgreen}{36.93} \\
Palo-13B
& 50.00 & 49.37 & 49.60 & 49.75 & 49.80 & 51.10 & 49.70
&
& 31.91 & 21.22 & 28.28 & 29.39 & 26.75 & 32.85 & 43.78 & 59.08 & 30.60 \\
DiM\textsuperscript{3}-7B
& \cellcolor{bestblue}\textbf{64.80} & \cellcolor{bestblue}\textbf{79.42} & \cellcolor{bestblue}\textbf{54.75} & \cellcolor{bestblue}\textbf{54.24} & \cellcolor{bestblue}\textbf{61.36} & \cellcolor{secondgreen}{54.05} & \cellcolor{bestblue}\textbf{62.91}
&
& \cellcolor{secondgreen}{38.96} & \cellcolor{bestblue}\textbf{51.98} & \cellcolor{bestblue}\textbf{49.21} & \cellcolor{bestblue}\textbf{48.00} & \cellcolor{bestblue}\textbf{49.80} & \cellcolor{bestblue}\textbf{47.93} & \cellcolor{bestblue}\textbf{50.42} & 59.27 & \cellcolor{bestblue}\textbf{48.04} \\
\bottomrule
\end{tabular}
\end{adjustbox}
\end{table*}

\begin{table*}[t]
\centering
\scriptsize
\setlength{\tabcolsep}{3pt}
\renewcommand{\arraystretch}{1.08}
\caption{Results on Afri-MCQA across multilingual settings.}
\label{tab:multilingual_results}
\resizebox{\textwidth}{!}{
\begin{tabular}{lccccccccccccccccc}
\toprule
\textbf{Method} & \textbf{Am} & \textbf{Ha} & \textbf{Ig} & \textbf{Lug} & \textbf{Om} & \textbf{Rw} & \textbf{Ki} & \textbf{So} & \textbf{Ti} & \textbf{Ak} & \textbf{Yo} & \textbf{Tn} & \textbf{Ny} & \textbf{Zu} & \textbf{Sot} & \textbf{Ln} & \textbf{Avg} \\
\midrule
\multicolumn{18}{c}{\textit{Llama-2-7B}} \\
\midrule
LLaVA-v1.5-7B & 0.216 & \cellcolor{bestblue}\textbf{0.256} & 0.256 & 0.265 & 0.225 & 0.295 & 0.237 & \cellcolor{secondgreen}0.270 & \cellcolor{secondgreen}0.199 & 0.276 & \cellcolor{bestblue}\textbf{0.235} & \cellcolor{secondgreen}0.255 & 0.200 & \cellcolor{bestblue}\textbf{0.260} & \cellcolor{secondgreen}0.254 & 0.237 & \cellcolor{secondgreen}0.246 \\
Task Arithmetic & 0.216 & \cellcolor{bestblue}\textbf{0.256} & 0.251 & 0.265 & 0.225 & 0.290 & 0.237 & \cellcolor{secondgreen}0.270 & \cellcolor{secondgreen}0.199 & 0.271 & \cellcolor{bestblue}\textbf{0.235} & 0.245 & 0.200 & \cellcolor{secondgreen}0.250 & \cellcolor{secondgreen}0.254 & 0.237 & 0.244 \\
DARE & 0.216 & \cellcolor{secondgreen}0.251 & 0.246 & 0.260 & 0.225 & 0.295 & 0.237 & \cellcolor{secondgreen}0.270 & \cellcolor{secondgreen}0.199 & 0.276 & \cellcolor{bestblue}\textbf{0.235} & \cellcolor{bestblue}\textbf{0.260} & 0.195 & 0.240 & \cellcolor{secondgreen}0.254 & 0.232 & 0.243 \\
TIES-Merging & 0.216 & \cellcolor{secondgreen}0.251 & 0.241 & 0.265 & 0.215 & 0.275 & 0.227 & 0.260 & \cellcolor{secondgreen}0.199 & 0.271 & \cellcolor{bestblue}\textbf{0.235} & 0.245 & 0.195 & 0.240 & 0.246 & 0.232 & 0.238 \\
Breadcrumbs & 0.060 & 0.246 & 0.236 & 0.250 & 0.210 & 0.265 & 0.232 & 0.260 & 0.046 & 0.229 & 0.215 & 0.240 & 0.180 & 0.229 & 0.223 & 0.237 & 0.210 \\
PCB-Merging & 0.000 & 0.000 & 0.000 & 0.000 & 0.000 & 0.000 & 0.000 & 0.000 & 0.000 & 0.000 & 0.000 & 0.000 & 0.000 & 0.000 & 0.000 & 0.000 & 0.000 \\
STF & \cellcolor{bestblue}\textbf{0.251} & 0.246 & 0.221 & 0.280 & \cellcolor{secondgreen}0.235 & \cellcolor{secondgreen}0.300 & 0.242 & 0.235 & \cellcolor{secondgreen}0.199 & 0.271 & 0.170 & \cellcolor{bestblue}\textbf{0.260} & \cellcolor{secondgreen}0.205 & 0.208 & 0.246 & 0.248 & 0.239 \\
NeuroMerging & 0.166 & 0.246 & \cellcolor{secondgreen}0.266 & \cellcolor{secondgreen}0.295 & 0.215 & \cellcolor{secondgreen}0.300 & \cellcolor{bestblue}\textbf{0.253} & \cellcolor{secondgreen}0.270 & 0.167 & 0.257 & 0.210 & 0.235 & 0.200 & 0.208 & 0.246 & \cellcolor{secondgreen}0.253 & 0.237 \\
Palo-7B & 0.216 & \cellcolor{secondgreen}0.251 & 0.261 & 0.255 & 0.225 & 0.290 & \cellcolor{secondgreen}0.247 & \cellcolor{secondgreen}0.270 & \cellcolor{secondgreen}0.199 & \cellcolor{secondgreen}0.280 & 0.225 & 0.250 & \cellcolor{secondgreen}0.205 & \cellcolor{bestblue}\textbf{0.260} & 0.246 & 0.232 & 0.245 \\
Palo-13B & 0.216 & \cellcolor{bestblue}\textbf{0.256} & 0.246 & 0.260 & 0.220 & 0.270 & 0.242 & 0.260 & \cellcolor{secondgreen}0.199 & 0.262 & \cellcolor{secondgreen}0.230 & 0.250 & 0.195 & 0.240 & \cellcolor{secondgreen}0.254 & 0.232 & 0.240 \\
DiM\textsuperscript{3}-7B & \cellcolor{secondgreen}0.241 & \cellcolor{bestblue}\textbf{0.256} & \cellcolor{bestblue}\textbf{0.281} & \cellcolor{bestblue}\textbf{0.320} & \cellcolor{bestblue}\textbf{0.245} & \cellcolor{bestblue}\textbf{0.335} & 0.232 & \cellcolor{bestblue}\textbf{0.320} & \cellcolor{bestblue}\textbf{0.204} & \cellcolor{bestblue}\textbf{0.304} & \cellcolor{bestblue}\textbf{0.235} & 0.250 & \cellcolor{bestblue}\textbf{0.215} & 0.240 & \cellcolor{bestblue}\textbf{0.269} & \cellcolor{bestblue}\textbf{0.263} & \cellcolor{bestblue}\textbf{0.263} \\
\midrule
\multicolumn{18}{c}{\textit{Qwen2-7B-Instruct}} \\
\midrule
Pangea-7B & \cellcolor{secondgreen}0.357 & 0.271 & 0.352 & 0.320 & 0.130 & 0.370 & 0.253 & 0.310 & 0.287 & 0.332 & 0.380 & 0.185 & 0.265 & 0.198 & 0.131 & 0.389 & 0.283 \\
Task Arithmetic & 0.332 & 0.312 & 0.357 & \cellcolor{bestblue}\textbf{0.425} & 0.240 & \cellcolor{secondgreen}0.475 & 0.345 & 0.340 & 0.306 & \cellcolor{secondgreen}0.360 & 0.390 & 0.265 & 0.345 & \cellcolor{bestblue}\textbf{0.333} & 0.285 & 0.424 & 0.346 \\
DARE & \cellcolor{bestblue}\textbf{0.367} & 0.292 & 0.322 & 0.350 & 0.225 & 0.430 & 0.335 & 0.320 & 0.301 & 0.322 & 0.320 & 0.265 & 0.340 & 0.271 & 0.231 & 0.404 & 0.318 \\
TIES-Merging & 0.337 & \cellcolor{secondgreen}0.327 & 0.357 & \cellcolor{secondgreen}0.415 & 0.275 & \cellcolor{bestblue}\textbf{0.480} & 0.366 & 0.335 & 0.301 & 0.351 & \cellcolor{secondgreen}0.395 & 0.315 & 0.385 & \cellcolor{bestblue}\textbf{0.333} & 0.292 & \cellcolor{bestblue}\textbf{0.455} & 0.357 \\
Breadcrumbs & 0.327 & 0.276 & 0.332 & 0.395 & 0.295 & 0.455 & 0.351 & \cellcolor{secondgreen}0.345 & 0.324 & 0.346 & \cellcolor{bestblue}\textbf{0.440} & \cellcolor{secondgreen}0.335 & \cellcolor{secondgreen}0.395 & 0.302 & \cellcolor{secondgreen}0.315 & 0.414 & 0.353 \\
PCB-Merging & 0.337 & 0.271 & 0.327 & 0.375 & \cellcolor{bestblue}\textbf{0.325} & 0.455 & \cellcolor{secondgreen}0.371 & \cellcolor{secondgreen}0.345 & 0.333 & 0.346 & 0.390 & 0.315 & 0.385 & \cellcolor{secondgreen}0.323 & 0.277 & \cellcolor{secondgreen}0.429 & 0.350 \\
STF & \cellcolor{bestblue}\textbf{0.367} & 0.312 & \cellcolor{secondgreen}0.387 & 0.380 & 0.260 & 0.400 & 0.304 & 0.325 & 0.333 & 0.318 & 0.350 & 0.265 & 0.275 & 0.292 & 0.208 & 0.394 & 0.323 \\
NeuroMerging & 0.307 & 0.281 & \cellcolor{bestblue}\textbf{0.407} & \cellcolor{secondgreen}0.415 & 0.305 & 0.455 & 0.340 & \cellcolor{bestblue}\textbf{0.380} & \cellcolor{secondgreen}0.343 & 0.346 & 0.380 & \cellcolor{bestblue}\textbf{0.355} & \cellcolor{bestblue}\textbf{0.400} & \cellcolor{bestblue}\textbf{0.333} & 0.308 & 0.424 & \cellcolor{secondgreen}0.361 \\
DiM\textsuperscript{3} & 0.317 & \cellcolor{bestblue}\textbf{0.342} & \cellcolor{bestblue}\textbf{0.407} & \cellcolor{bestblue}\textbf{0.425} & \cellcolor{secondgreen}0.320 & \cellcolor{bestblue}\textbf{0.480} & \cellcolor{bestblue}\textbf{0.376} & \cellcolor{secondgreen}0.345 & \cellcolor{bestblue}\textbf{0.347} & \cellcolor{bestblue}\textbf{0.388} & 0.390 & 0.315 & 0.390 & \cellcolor{bestblue}\textbf{0.333} & \cellcolor{bestblue}\textbf{0.331} & 0.419 & \cellcolor{bestblue}\textbf{0.370} \\
\midrule
\multicolumn{18}{c}{\textit{Qwen3-8B}} \\
\midrule
Qwen3-VL-8B-Instruct & 0.377 & 0.342 & 0.347 & 0.415 & 0.335 & 0.460 & \cellcolor{secondgreen}0.392 & 0.405 & 0.357 & \cellcolor{bestblue}\textbf{0.430} & 0.420 & \cellcolor{secondgreen}0.350 & 0.380 & 0.302 & \cellcolor{secondgreen}0.323 & \cellcolor{secondgreen}0.490 & 0.383 \\
Task Arithmetic & 0.402 & 0.337 & 0.352 & \cellcolor{secondgreen}0.420 & 0.340 & \cellcolor{secondgreen}0.495 & 0.381 & 0.435 & 0.357 & 0.383 & 0.415 & 0.335 & 0.385 & 0.281 & 0.308 & \cellcolor{bestblue}\textbf{0.505} & 0.383 \\
DARE & 0.231 & 0.181 & 0.236 & 0.220 & 0.245 & 0.190 & 0.289 & 0.250 & 0.273 & 0.224 & 0.265 & 0.225 & 0.280 & 0.250 & 0.300 & 0.248 & 0.244 \\
TIES-Merging & 0.382 & 0.342 & \cellcolor{bestblue}\textbf{0.372} & 0.410 & 0.390 & 0.490 & \cellcolor{bestblue}\textbf{0.407} & 0.425 & 0.366 & 0.383 & 0.415 & 0.335 & 0.395 & 0.292 & 0.269 & \cellcolor{bestblue}\textbf{0.505} & 0.386 \\
Breadcrumbs & 0.010 & 0.015 & 0.025 & 0.000 & 0.000 & 0.015 & 0.000 & 0.000 & 0.028 & 0.000 & 0.000 & 0.005 & 0.025 & 0.000 & 0.000 & 0.020 & 0.009 \\
PCB-Merging & 0.402 & 0.342 & 0.307 & 0.395 & 0.370 & 0.475 & 0.366 & 0.430 & 0.426 & 0.374 & 0.425 & 0.320 & 0.370 & 0.313 & \cellcolor{bestblue}\textbf{0.346} & 0.404 & 0.379 \\
STF & \cellcolor{secondgreen}0.427 & \cellcolor{secondgreen}0.442 & \cellcolor{secondgreen}0.367 & 0.400 & \cellcolor{secondgreen}0.415 & \cellcolor{bestblue}\textbf{0.645} & \cellcolor{bestblue}\textbf{0.407} & \cellcolor{secondgreen}0.530 & \cellcolor{secondgreen}0.458 & 0.336 & \cellcolor{bestblue}\textbf{0.475} & 0.340 & \cellcolor{secondgreen}0.440 & \cellcolor{secondgreen}0.354 & 0.262 & 0.404 & \cellcolor{secondgreen}0.419 \\
NeuroMerging & 0.211 & 0.292 & 0.226 & 0.240 & 0.275 & 0.380 & 0.191 & 0.385 & 0.273 & 0.215 & 0.335 & 0.235 & 0.235 & 0.281 & 0.200 & 0.187 & 0.260 \\
DiM\textsuperscript{3} & \cellcolor{bestblue}\textbf{0.493} & \cellcolor{bestblue}\textbf{0.498} & 0.362 & \cellcolor{bestblue}\textbf{0.450} & \cellcolor{bestblue}\textbf{0.445} & \cellcolor{bestblue}\textbf{0.645} & \cellcolor{bestblue}\textbf{0.407} & \cellcolor{bestblue}\textbf{0.555} & \cellcolor{bestblue}\textbf{0.505} & \cellcolor{secondgreen}0.388 & \cellcolor{secondgreen}0.445 & \cellcolor{bestblue}\textbf{0.420} & \cellcolor{bestblue}\textbf{0.480} & \cellcolor{bestblue}\textbf{0.406} & 0.277 & 0.460 & \cellcolor{bestblue}\textbf{0.452} \\
\bottomrule
\end{tabular}
}
\end{table*}

\begin{table*}[t]
\centering
\caption{Multilingual performance comparison of different merging methods on CVQA.}
\label{tab:cvqa}
\resizebox{\textwidth}{!}{
\begin{tabular}{lccccccccccc}
\toprule
\textbf{Language} & \textbf{LLaVA-v1.5-7B} & \textbf{Task Arithmetic} & \textbf{DARE} & \textbf{TIES-Merging} & \textbf{Breadcrumbs} & \textbf{PCB-Merging} & \textbf{STF} & \textbf{NeuroMerging} & \textbf{Palo-7B} & \textbf{Palo-13B} & \textbf{DiM\textsuperscript{3}} \\
\midrule
Am  & 26.50 & 29.06 & 29.91 & 27.78 & 25.21 & 0.43  & \cellcolor{bestblue}\textbf{35.90} & 30.34 & 19.23 & 26.07 & \cellcolor{secondgreen}33.76 \\
Bn  & 29.37 & 31.47 & 31.47 & 32.87 & 29.02 & 1.05  & 23.43 & 25.52 & \cellcolor{bestblue}\textbf{37.76} & \cellcolor{secondgreen}36.01 & 32.52 \\
Br  & 29.63 & 31.36 & 31.11 & 29.38 & 27.16 & 20.74 & 23.95 & \cellcolor{bestblue}\textbf{33.58} & 29.14 & \cellcolor{secondgreen}33.33 & 30.86 \\
Bg  & 39.35 & 39.08 & 40.43 & 37.74 & 34.23 & 11.05 & 34.77 & \cellcolor{bestblue}\textbf{42.05} & 36.66 & \cellcolor{secondgreen}40.97 & 39.89 \\
Zh  & \cellcolor{secondgreen}48.57 & \cellcolor{bestblue}\textbf{48.95} & 45.89 & 47.80 & 41.87 & 0.19  & 44.17 & 47.23 & 45.51 & 46.65 & 47.80 \\
Ar  & 32.51 & \cellcolor{secondgreen}35.96 & 34.98 & 33.00 & 30.05 & 2.46  & 32.51 & 33.99 & 29.06 & 27.59 & \cellcolor{bestblue}\textbf{40.39} \\
Fil & 46.31 & \cellcolor{secondgreen}46.80 & 45.32 & 45.32 & 37.44 & 5.91  & 37.44 & \cellcolor{bestblue}\textbf{47.29} & 38.92 & \cellcolor{secondgreen}46.80 & \cellcolor{bestblue}\textbf{47.29} \\
Hi  & 39.80 & 37.31 & 38.81 & 39.80 & 31.34 & 1.49  & 45.27 & \cellcolor{bestblue}\textbf{51.74} & 42.79 & 45.27 & \cellcolor{secondgreen}51.24 \\
Ig  & 32.50 & \cellcolor{bestblue}\textbf{35.50} & \cellcolor{secondgreen}34.00 & 33.00 & 28.00 & 8.00  & 12.00 & 28.00 & 29.50 & 33.00 & \cellcolor{bestblue}\textbf{35.50} \\
Id  & 42.23 & 46.12 & 42.23 & \cellcolor{secondgreen}46.60 & 39.56 & 5.83  & 36.17 & \cellcolor{secondgreen}46.60 & 41.26 & 44.90 & \cellcolor{bestblue}\textbf{49.03} \\
Ga  & 43.87 & \cellcolor{secondgreen}47.24 & 44.48 & \cellcolor{bestblue}\textbf{48.16} & 38.04 & 7.06  & 34.97 & 45.09 & 43.25 & 37.12 & \cellcolor{secondgreen}47.24 \\
Ja  & 34.98 & \cellcolor{secondgreen}36.45 & 32.02 & 33.00 & 28.08 & 0.00  & 31.03 & 31.53 & 30.54 & \cellcolor{bestblue}\textbf{37.93} & 32.51 \\
Jv  & 36.36 & 39.06 & 32.66 & 36.36 & 29.97 & 11.78 & 29.29 & \cellcolor{secondgreen}39.39 & 32.32 & 38.72 & \cellcolor{bestblue}\textbf{42.42} \\
Rw  & 34.47 & \cellcolor{secondgreen}35.74 & 32.34 & 34.04 & 29.79 & 20.00 & 32.34 & 33.62 & 28.94 & 28.94 & \cellcolor{bestblue}\textbf{43.40} \\
Ko  & \cellcolor{bestblue}\textbf{51.03} & \cellcolor{bestblue}\textbf{51.03} & \cellcolor{secondgreen}50.69 & \cellcolor{bestblue}\textbf{51.03} & 46.55 & 0.00  & 47.24 & 43.79 & 44.48 & 46.55 & 47.93 \\
Ms  & 47.94 & \cellcolor{secondgreen}51.11 & 47.30 & \cellcolor{bestblue}\textbf{51.75} & 40.32 & 9.52  & 40.95 & 44.76 & 42.86 & 46.35 & 48.57 \\
Mr  & 36.14 & 36.14 & 35.64 & 34.16 & 26.24 & 0.99  & \cellcolor{secondgreen}42.08 & 36.63 & 39.11 & 39.11 & \cellcolor{bestblue}\textbf{45.05} \\
Min & 36.65 & 37.05 & 33.07 & 34.66 & 33.86 & 9.56  & 29.48 & \cellcolor{secondgreen}41.04 & 31.87 & 39.04 & \cellcolor{bestblue}\textbf{42.63} \\
Mn  & 29.49 & 29.17 & 28.53 & \cellcolor{secondgreen}32.37 & 27.24 & 19.55 & 30.13 & 30.13 & 28.53 & 26.92 & \cellcolor{bestblue}\textbf{36.22} \\
No  & 52.51 & \cellcolor{bestblue}\textbf{53.51} & 51.17 & 50.17 & 42.14 & 7.36  & 43.81 & 51.17 & 48.83 & \cellcolor{secondgreen}53.18 & 50.50 \\
Om  & \cellcolor{bestblue}\textbf{37.38} & 35.05 & 34.58 & 32.24 & 29.44 & 18.22 & 17.29 & 33.18 & 32.71 & 31.78 & \cellcolor{secondgreen}35.98 \\
Pt  & \cellcolor{secondgreen}57.39 & \cellcolor{bestblue}\textbf{58.80} & 56.69 & 55.28 & 52.46 & 17.61 & 47.54 & \cellcolor{secondgreen}57.39 & 56.69 & 56.69 & 53.52 \\
Ro  & 50.99 & 53.64 & 51.32 & \cellcolor{bestblue}\textbf{55.96} & 46.69 & 3.31  & 44.70 & 50.99 & 46.03 & \cellcolor{secondgreen}55.63 & 53.64 \\
Ru  & 57.50 & 57.50 & 52.50 & 55.50 & 46.50 & 8.00  & 56.50 & \cellcolor{secondgreen}60.00 & 47.00 & 52.50 & \cellcolor{bestblue}\textbf{61.00} \\
Si  & 24.44 & 25.33 & 27.11 & \cellcolor{secondgreen}28.00 & 26.22 & 2.22  & 9.33  & 11.56 & 27.56 & 24.89 & \cellcolor{bestblue}\textbf{39.11} \\
Es  & \cellcolor{secondgreen}54.57 & \cellcolor{bestblue}\textbf{54.71} & 52.43 & 52.19 & 47.72 & 18.90 & 44.22 & 52.19 & 49.66 & 51.51 & 51.60 \\
Su  & 38.00 & 38.50 & 34.50 & 38.00 & 33.00 & 8.50  & 27.00 & \cellcolor{secondgreen}41.00 & 32.00 & 40.50 & \cellcolor{bestblue}\textbf{44.50} \\
Sw  & 36.26 & 39.56 & 37.36 & 38.10 & 33.33 & 18.68 & 35.90 & \cellcolor{secondgreen}47.25 & 36.26 & 40.66 & \cellcolor{bestblue}\textbf{49.45} \\
Ta  & 31.78 & 30.84 & 27.57 & 30.84 & 28.50 & 0.93  & \cellcolor{secondgreen}34.11 & 28.50 & 27.57 & 32.71 & \cellcolor{bestblue}\textbf{40.65} \\
Te  & \cellcolor{secondgreen}34.50 & 32.50 & 30.00 & 25.50 & 25.00 & 0.00  & 26.50 & 16.00 & 28.50 & 26.50 & \cellcolor{bestblue}\textbf{36.50} \\
Ur  & 25.92 & 30.05 & 25.00 & 28.67 & 26.38 & 2.52  & 42.20 & 39.22 & \cellcolor{secondgreen}42.43 & 39.45 & \cellcolor{bestblue}\textbf{46.33} \\
\midrule
$\mathrm{Avg}_{\mathrm{mul}}$ & 39.32 & \cellcolor{secondgreen}40.47 & 38.42 & 39.33 & 34.24 & 7.80 & 34.59 & 39.38 & 37.00 & 39.59 & \cellcolor{bestblue}\textbf{43.78} \\
\bottomrule
\end{tabular}
}
\end{table*}

\begin{table}[t]
\centering
\caption{Results on three general multimodal benchmarks: MMStar, MMMU, and SEED-Bench-2-Plus. This table reports the exact scores corresponding to the comparison summarized in Figure~\ref{fig:general_mm_results}.}
\label{tab:mmstar_mmmu_seedbench}
\small
\renewcommand{\arraystretch}{0.95}
\setlength{\tabcolsep}{5pt}
\begin{tabular}{lccc}
\toprule
\textbf{Models} & \textbf{MMStar} & \textbf{MMMU} & \textbf{SEED-Bench-2-Plus} \\
\specialrule{0.08em}{0.15em}{0.1em}
LLaVA-v1.5-7B   & 33.83 & 36.33 & 40.97 \\
Task Arithmetic & 33.53 & \cellcolor{secondgreen}36.89 & \cellcolor{secondgreen}42.07 \\
DARE            & \cellcolor{bestblue}\textbf{36.11} & 35.78 & 41.90 \\
TIES-Merging    & 34.27 & 36.33 & \cellcolor{bestblue}\textbf{43.13} \\
Breadcrumbs     & 32.01 & 33.33 & 40.32 \\
STF             & 31.51 & 30.67 & 34.43 \\
NeuroMerging    & 34.58 & \cellcolor{secondgreen}36.89 & 40.84 \\
Palo-7B         & 32.95 & 33.11 & 38.08 \\
DiM\textsuperscript{3}-7B            & \cellcolor{secondgreen}34.86 & \cellcolor{bestblue}\textbf{37.22} & 40.45 \\
\bottomrule
\end{tabular}
\end{table}

\begin{figure}[!h]
  \centering
  \includegraphics[width=\linewidth]{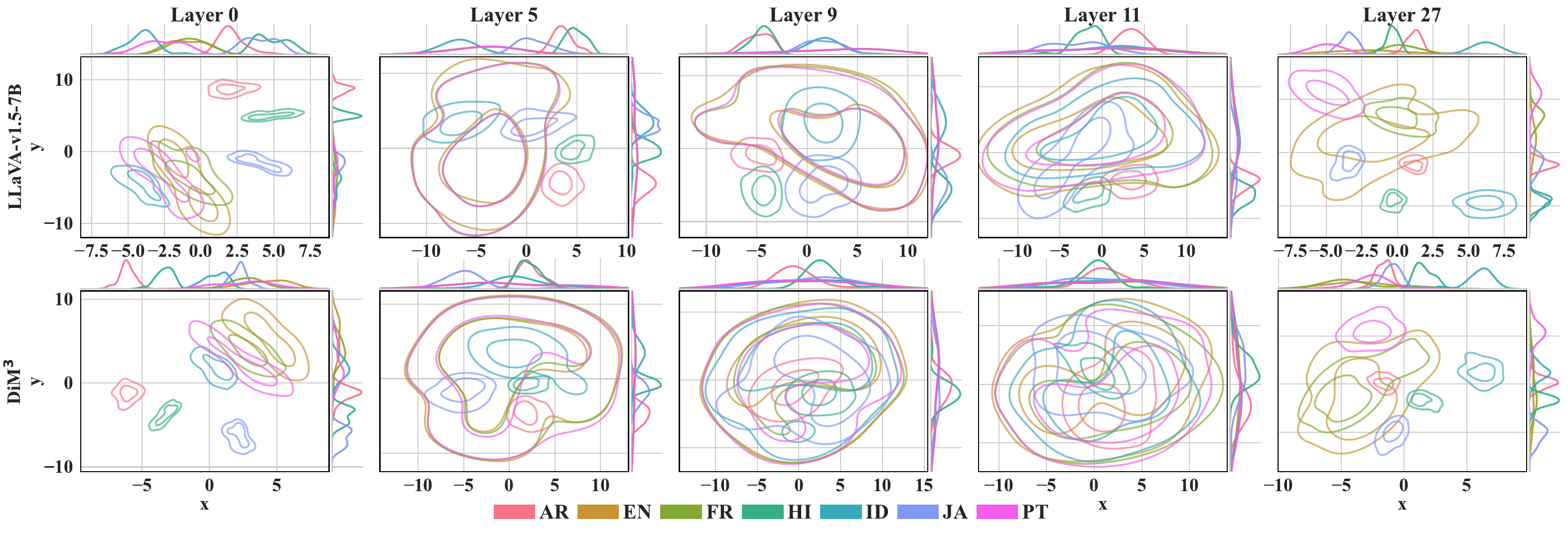}
  \caption{t-SNE visualizations of average-pooled hidden states for the question spans in multilingual multimodal inputs from xMMMU across selected layers. We focus on question tokens to avoid confounding effects from English instruction templates and image tokens.}
  \label{fig:xmmmu_tsne}
\end{figure}

\begin{figure}[!h]
  \centering
  \includegraphics[width=\linewidth]{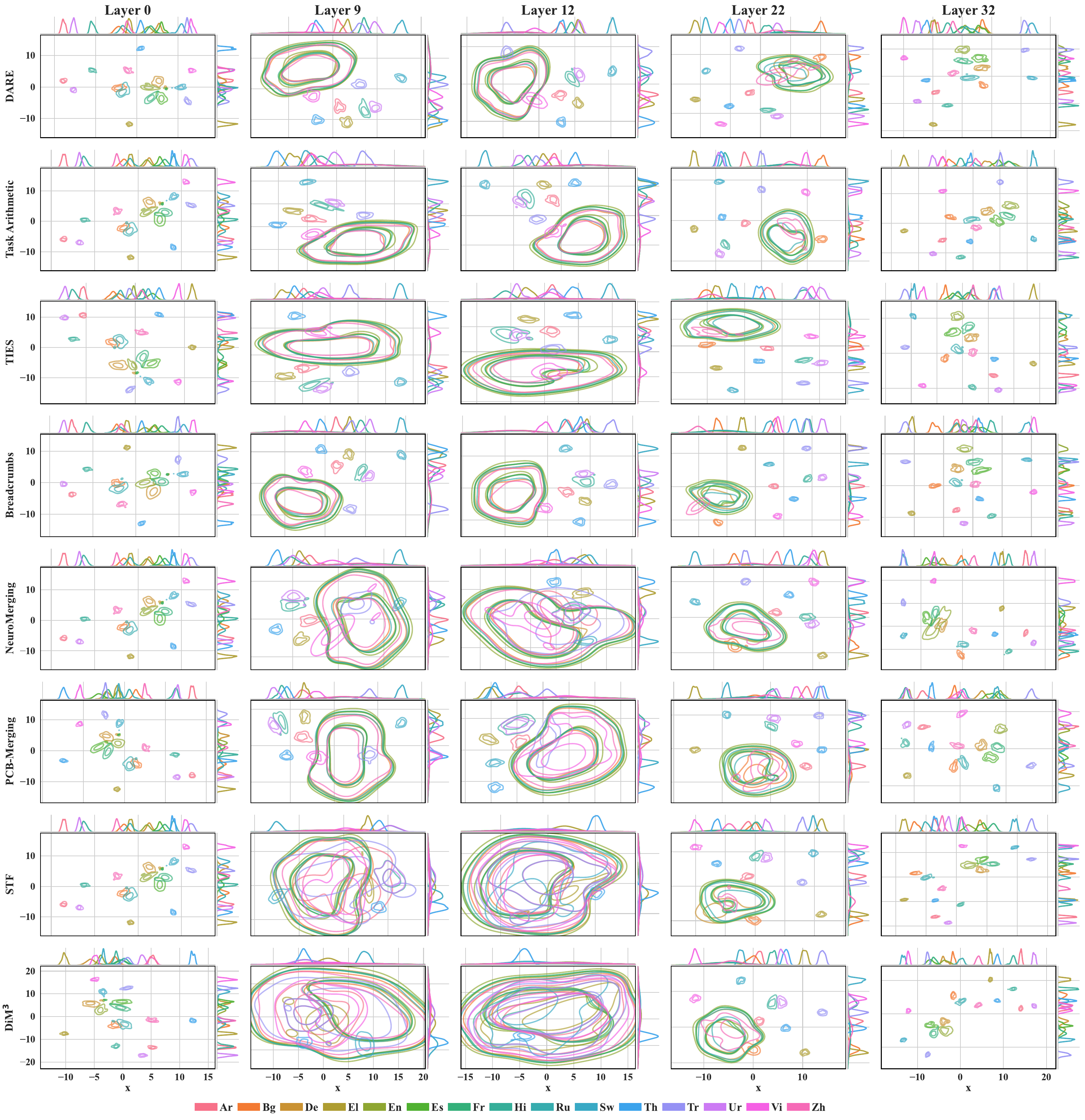}
  \caption{Additional t-SNE visualizations of average-pooled hidden states on multilingual text inputs from XNLI across selected layers. The figure compares a broader set of merging baselines to illustrate differences in cross-lingual hidden-state geometry. }
  \label{fig:text_only_tsne_merging_methods}
\end{figure}

\begin{figure}[!h]
  \centering
  \includegraphics[width=0.9\linewidth]{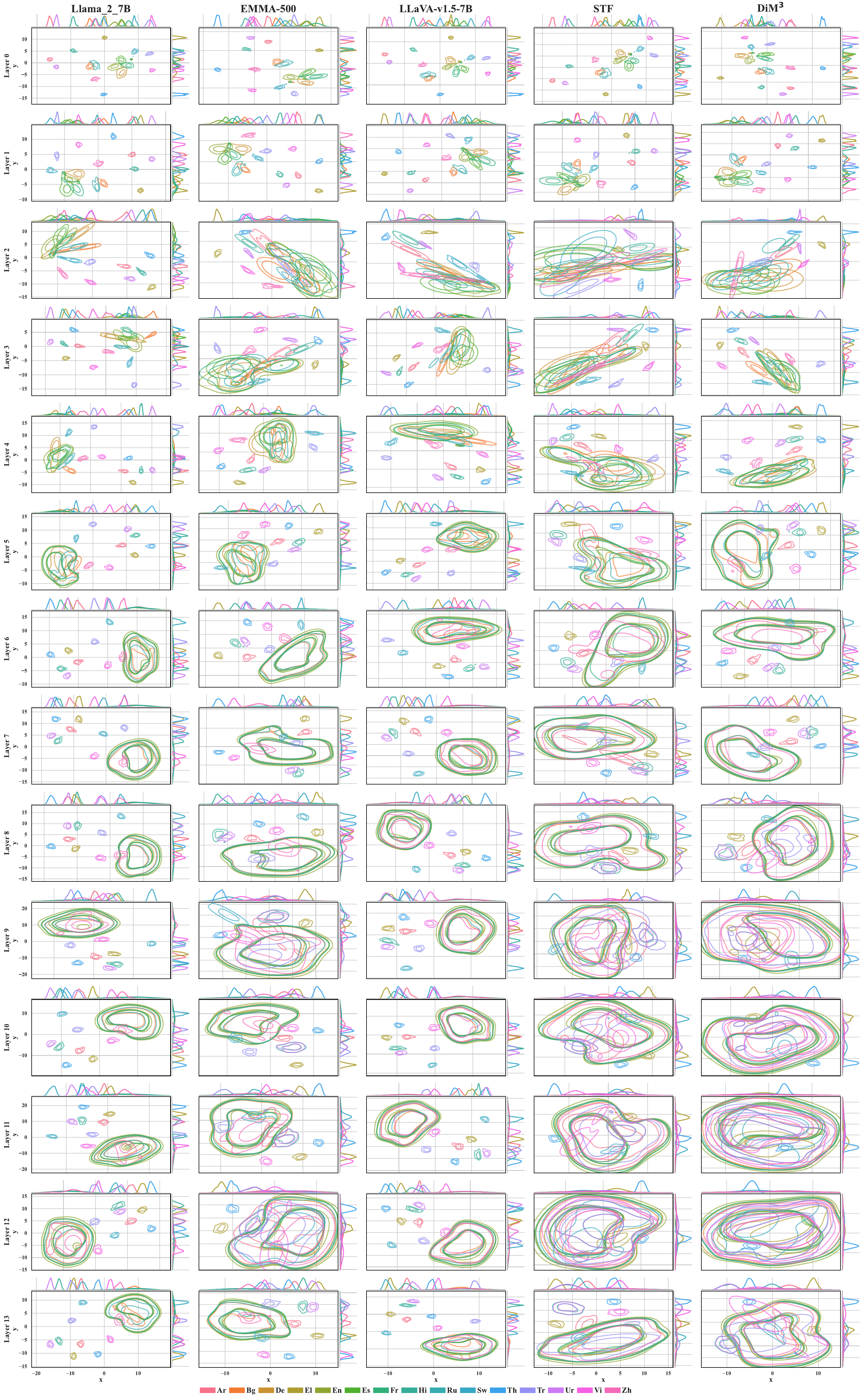}
  \caption{Full layer-wise t-SNE visualizations of average-pooled hidden states on multilingual text inputs from XNLI for layers 0--13.}
  \label{fig:text_only_tsne_0-13}
\end{figure}

\begin{figure}[!h]
  \centering
  \includegraphics[width=0.9\linewidth]{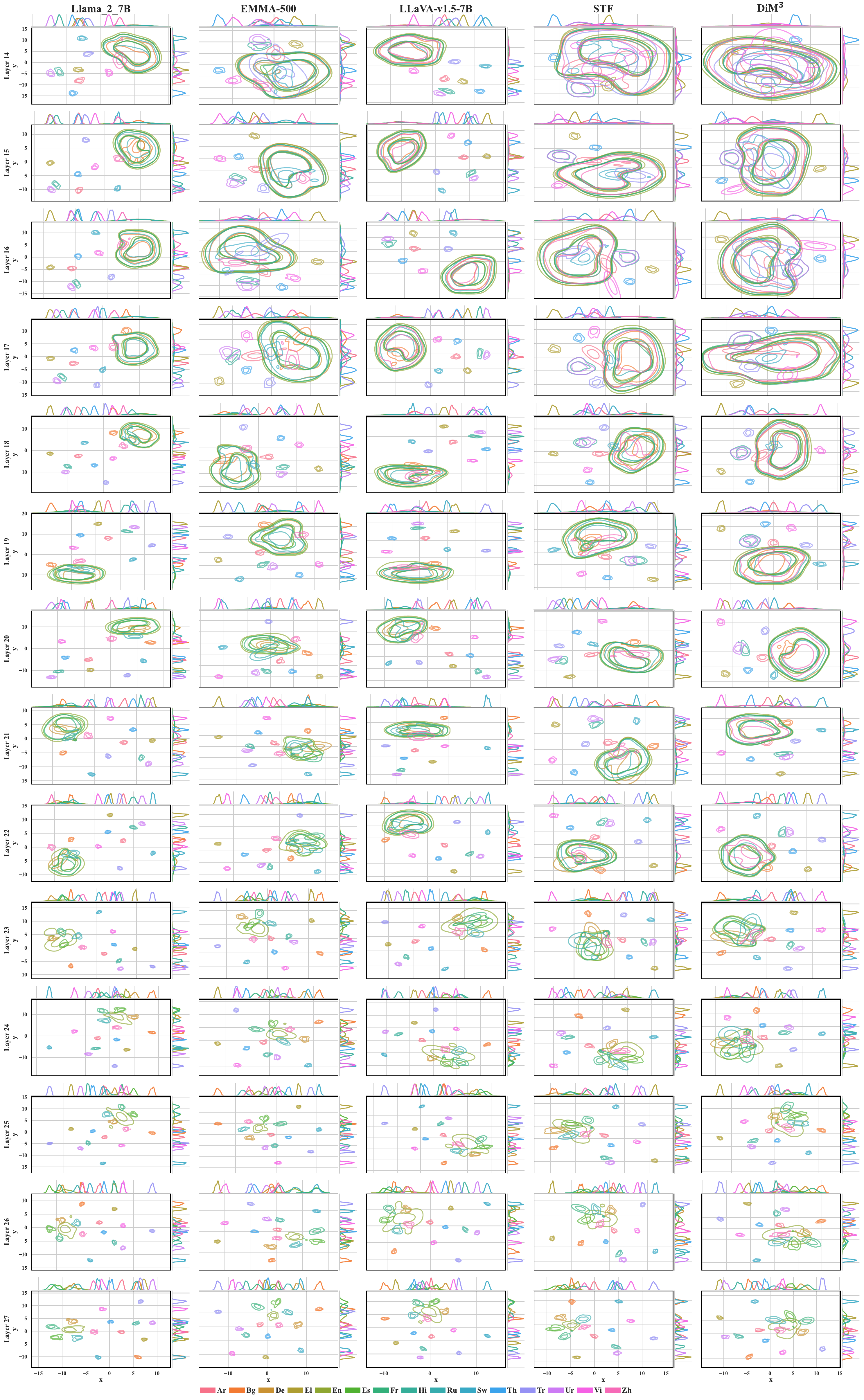}
  \caption{Full layer-wise t-SNE visualizations of average-pooled hidden states on multilingual text inputs from XNLI for layers 14--27.}
  \label{fig:text_only_tsne_14-27}
\end{figure}

\begin{figure}[!h]
  \centering
  \includegraphics[width=\linewidth]{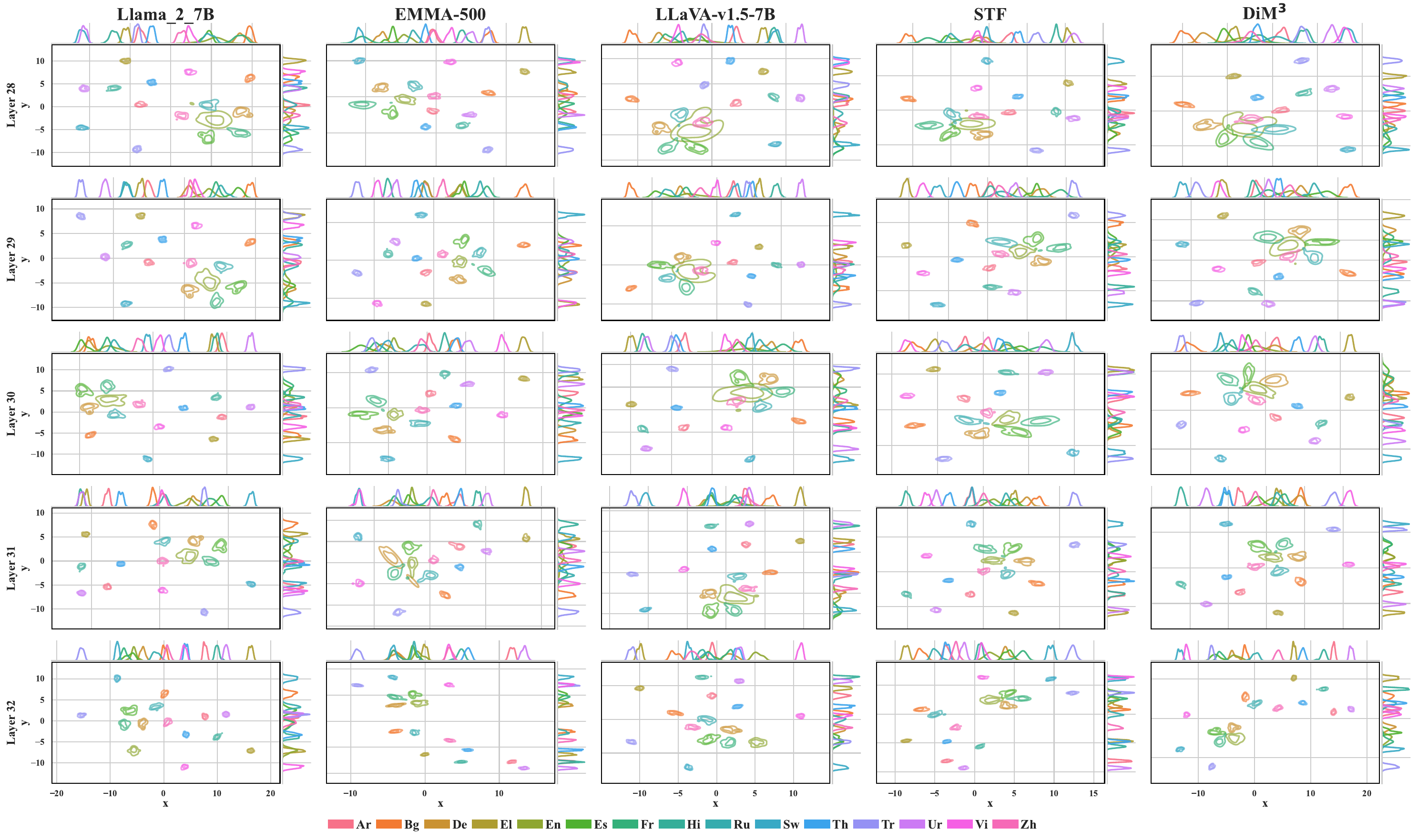}
  \caption{Full layer-wise t-SNE visualizations of average-pooled hidden states on multilingual text inputs from XNLI for layers 28--32.}
  \label{fig:text_only_tsne_28-32}
\end{figure}

\clearpage
\newpage

\end{document}